\documentclass[10pt,twocolumn,letterpaper]{article}

\usepackage{times}
\usepackage{epsfig}
\usepackage{graphicx}
\usepackage{amsmath}
\usepackage{amssymb}
\usepackage{amsthm}
\usepackage{bm} 
\usepackage{makecell}

\usepackage{mathtools}

\newtheorem{theorem}{Theorem}
\newtheorem{definition}{Definition}
\newtheorem{prop}{Proposition}
\newtheorem{lemma}{Lemma}
\newtheorem{example}{Example}
\newtheorem{remark}{Remark}

\newtheorem{cor}{Corollary}

\newcommand{\mb}[1]{\mathbf{#1}}
\renewcommand{\xi}{\mathbf{x}^{i}}

\newcommand{\E}{\mathop{\mathbb{\Large E}}}
\newcommand{\Eeps}{\ensuremath{\E_{\varepsilon(\xi)}}}
\newcommand{\T}{\ensuremath{\top}}

\newcommand{\Unit}{\mathcal{I}}
\newcommand{\Lrec}{L_{\text{rec}}}
\newcommand{\diagentry}[1]{\mathmakebox[1.8em]{#1}}
\DeclareMathOperator*{\argmax}{arg\,max}
\DeclareMathOperator*{\mi}{MI}

\newcommand{\proofof}[2]{
\noindent {\it Proof of #1.}\hspace{0.1em} #2
\qed
\medskip
}

\usepackage{xspace}
\makeatletter
\DeclareRobustCommand\onedot{\futurelet\@let@token\@onedot}
\def\@onedot{\ifx\@let@token.\else.\null\fi\xspace}
\def\iid{{i.i.d}\onedot}
\def\eg{{e.g}\onedot} 
\def\ie{{i.e}\onedot}

\makeatother

\DeclarePairedDelimiterX{\KL}[2]{D_{\mathrm{KL}}(}{)}{%
  #1\;\delimsize\|\;#2%
}

\DeclareMathOperator{\Enc}{Enc_{\varphi}}
\DeclareMathOperator{\Dec}{Dec_{\theta}}
\DeclareMathOperator{\diag}{diag}
\DeclareMathOperator{\var}{var}

\DeclareMathOperator{\col}{col_\Pi}
\DeclareMathOperator{\psdet}{det^\dagger}
\DeclareMathOperator{\tr}{tr}

\newcommand*\dd{\mathop{}\!\mathrm{d}}

\usepackage[pagebackref=true,breaklinks=true,letterpaper=true,colorlinks,bookmarks=false]{hyperref}

\begin{document}

\title{Variational Autoencoders Pursue PCA Directions (by Accident)}

\author{Michal Rol\'{i}nek\footnotemark[1]~, Dominik Zietlow\footnotemark[1]~ and Georg Martius\\
Max-Planck-Institute for Intelligent Systems, 
T\"ubingen, Germany\\
{\tt\small \{mrolinek, dzietlow, gmartius\}@tue.mpg.de}}

\maketitle

\begin{abstract}
The Variational Autoencoder (VAE) is a powerful architecture capable of representation learning and generative modeling. When it comes to learning interpretable (disentangled) representations, VAE and its variants show unparalleled performance.
However, the reasons for this are unclear, since a very particular alignment of the latent embedding is needed but the design of the VAE does not encourage it in any explicit way.
We address this matter and offer the following explanation: the diagonal approximation in the encoder together with the inherent stochasticity force local orthogonality of the decoder. The local behavior of promoting both reconstruction and orthogonality matches closely how the PCA embedding is chosen. Alongside providing an intuitive understanding, we justify the statement with full theoretical analysis as well as with experiments. \end{abstract}

\footnotetext[1]{These authors contributed equally to this work.}
\section{Introduction}
The Variational Autoencoder (VAE)~\cite{KingmaWelling2014:VAE, rezende2014stochastic} is one of the foundational architectures in modern-day deep learning.
It serves both as a generative model as well as a representation learning technique.
The generative model is predominantely exploited in computer vision~\cite{dcign, 2016arXiv160408772G, InvAutoFlow, GregorEtAl2015:DRAWImgGen} with notable exceptions
such as generating combinatorial graphs~\cite{kusner2017grammar}.
As for representation learning, there is a variety of applications, ranging over image interpolation \cite{HouEtAl2017:ConsistentVA-interpolation}, one-shot generalization~\cite{RezendeEtAl2016:oneshotgeneralization}, language models~\cite{yangEtAl2017:VAE4Text}, speech transformation~\cite{blaauwBonada2016:speechtransformVAE}, and more.
Aside from direct applications, VAEs embody the success of variational methods in deep learning and have inspired a wide range of ongoing research~\cite{variational-flows, ZhangEtAl2017:AdvancesInVI}.

Recently, unsupervised learning of interpretable latent representations has received a lot of attention.
Interpretability of the latent code is an intuitively clear concept.
For instance, when representing faces one latent variable would solely correspond to the gender of the person, another to skin tone, yet another to hair color and so forth. Once such a representation is found it allows for interpretable latent code manipulation, which is desirable in a variety of applications; recently, for example, in reinforcement learning
\cite{DBLP:journals/corr/abs-1709-02349, 2017arXiv170708475H, 8206046, 7759578, 2018arXiv180704742N}.

The term \emph{disentanglement} \cite{DesjardinsXBengio2012:Disentfactors, Bengio12representationlearning, locatello2018challenging} offers a more formal approach. 
A representation is considered disentangled if each latent component encodes precisely one ``aspect'' (a generative factor) of the data. Under the current disentanglement metrics \cite{beta-vae, factor-vae, tc-beta-vae, locatello2018challenging}, VAE-based architectures ($\beta$-VAE~\cite{beta-vae}, TCVAE~\cite{tc-beta-vae}, FactorVAE~\cite{factor-vae}) dominate the benchmarks, leaving behind other approaches such as InfoGAN~\cite{infogan2016} and DCIGN~\cite{dcign}. Exemplarily, a latent traversal for a $\beta$-VAE is shown in Fig. \ref{fig:latent_trav} in which precisely one generative factor is isloated (face azimuth).

\begin{figure}
\centering{
\includegraphics[width=0.9\linewidth]{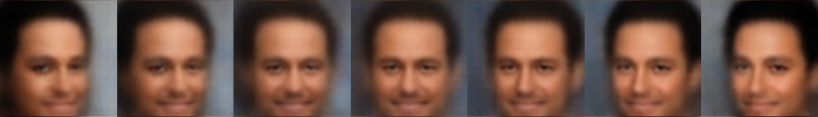}
\caption{Latent traversal over a single latent coordinate on an exemplary image from the CelebA dataset~\cite{liu2015faceattributes} for a trained $\beta$-VAE. The latent coordinate clearly {\bf isolates} the azimuth angle. Provided by courtesy of the authors of~\cite{beta-vae}.
}
}
\label{fig:latent_trav}
\end{figure}
The success of VAE-based architectures on disentanglement tasks comes with a certain surprise.
One surprising aspect is that VAEs have been challenged on both of its own design functionalities,
as generative models~\cite{gan-paper, 2019arXiv190312436G} and as log-likelihood optimizers~\cite{adv-auto, adv-var-bayes}. Yet, no such claims are made in terms of disentanglement.
Another surprise stems from the fact that disentanglement requires the following feature:
the representative low-dimensional manifold must be aligned well with the coordinate axes.
However, the design of the VAE does not suggest any such mechanism. On the contrary, the idealized log-likelihood objective is, for example, invariant to rotational changes in the alignment.

Such observations have planted a suspicion that the inner workings of the VAE are not sufficiently understood.
Several recent works approached this issue \cite{understanding-disent, anonymous2019isa-vae, hidden-talents, fixing-broken-elbo, 2019arXiv190312436G, 2018arXiv181202833M,2019arXiv190305789D}.
However, a mechanistic explanation for the VAE's unexpected ability to disentangle is still missing.

In this paper, we isolate an internal mechanism of the VAE (also $\beta$-VAE) responsible for
choosing a particular latent representation and its alignment.
We give theoretical analysis covering also the nonlinear case and explain the discovered dynamics intuitively.
We show that this mechanism promotes local orthogonality of the embedding transformation and clarify how this orthogonality corresponds to good disentanglement.
Further, we uncover strong resemblance between this mechanism and the classical Principle Components Analysis (PCA) algorithm.
We confirm our theoretical findings in experiments.

Our theoretical approach is particular in the following ways:
(a) we base the analysis on the \emph{implemented} loss function in contrast to the typically considered idealized loss,
and (b) we identify a specific regime, prevalent in practice, and utilize it for a vital simplification. This simplification is the crucial step in enabling formalization.

The results, other than being significant on their own, also provide a solid explanation of ``why $\beta$-VAEs disentangle''.

\section{Background}
\label{sec:background}
Let us begin with reviewing the basics of VAE, PCA, and of the Singular Value Decomposition (SVD), along with a more detailed overview of disentanglement.
\subsection{Variational Autoencoders}
Let $\{\xi\}_{i=1}^N$ be a dataset consisting of $N$ \iid samples $\xi \in X = \mathbb{R}^n$ of a random variable $\mb{x}$.
An autoencoder framework operates with two mappings, the encoder $\Enc \colon X \to Z$ and the decoder $\Dec \colon Z \to X$, where $Z = \mathbb{R}^d$ is called the {\it latent space}.
In case of the VAE, both mappings are probabilistic and a fixed {\it prior distribution} $p(\mb{z})$ over $Z$ is assumed.
Since the distribution of $\mb{x}$ is also fixed (actual data distribution $q(\mb{x})$), the mappings $\Enc$ and $\Dec$ induce joint distributions
$q(\mb x, \mb z)=q_{\varphi}(\mb z|\mb x) q(\mb x)$
and
$p(x, z) = p_{\theta}(\mb x| \mb z) p(\mb z)$, respectively
 (omitting the dependencies on parameters $\theta$ and $\varphi$).
 The idealized VAE objective is then the marginalized log-likelihood
\begin{align}
\sum_{i=1}^N \log p(\xi). \label{eq:loglhood}
\end{align}
This objective is, however, not tractable and is approximated by the evidence lower bound (ELBO) \cite{KingmaWelling2014:VAE}.
For a fixed $\xi$ the log-likelihood $\log p(\xi)$ is lower bounded by
\begin{align}
 \E_{\mb{z} \sim q(\mb{z} \mid \xi)}\log p(\xi \mid \mb{z}) - \KL{q(\mb{z}\mid \xi)}{p(\mb{z})}\,, \label{eq:raw_elbo}
\end{align}
where the first term corresponds to the reconstruction loss and the second to the KL divergence between
 the latent representation $q(\mb{z}\mid \xi)$ and the prior distribution $p(\mb{z})$.
A variant, the $\beta$-VAE~\cite{beta-vae}, introduces a weighting $\beta$ on the KL term for regulating the trade-off between reconstruction (first term) and the proximity to the prior. Our analysis will automatically cover this case as well.

Finally, the prior $p(\mb{z})$ is set to $\mathcal{N}(0, \mathcal{I})$ and the encoder is assumed to have the form
\begin{align}
  \Enc(\mb{x}) \sim q_{\varphi}(\mb z|\mb x) = \mathcal{N}\left(\mb{\mu_\varphi}(\mb{x}), \diag{\sigma^2_\varphi}(\mb{x})\right), \label{eq:diagonal_enc}
\end{align}
where $\mb{\mu}_\varphi$ and $\mb{\sigma}_\varphi$ are deterministic mappings depending on parameters $\varphi$. Note particularly, that {\bf the covariance matrix is enforced to be diagonal}.
This turns out to be highly significant for the main result of this work.
The KL-divergence in \eqref{eq:raw_elbo} can be computed in closed form as
\begin{align}
L_{\text{KL}} = \frac12 \sum_{j=1}^{d} \left( \mu^2_j(\xi) + \sigma^2_j(\xi) - \log\sigma^2_j(\xi) -1 \right).
\label{eq:obj_impl}
\end{align}
In practical implementations, the reconstruction term from \eqref{eq:raw_elbo} is approximated with either a square loss or a cross-entropy loss.
\subsection{Disentanglement}
In the context of learning interpretable representations \cite{Bengio12representationlearning,beta-vae,understanding-disent,anonymous2019isa-vae,Schmidhuber1992:LearningFactorialCodes}
 it is useful to assume that the data originates from a process with some generating factors.
For instance, for images of faces this could be face azimuth, skin brightness, hair length, and so on.
Disentangled representations can then be defined as ones in which individual latent variables are sensitive to changes in individual generating factors, while being relatively insensitive to other changes \cite{Bengio12representationlearning}. Although quantifying disentanglement is nontrivial, several metrics have been proposed \cite{factor-vae, beta-vae, tc-beta-vae}.

Note also, that disentanglement is impossible without first learning a sufficiently expressive latent representation capable of good reconstruction.

In an unsupervised setting, the generating factors are of course unknown and the learning has to
resort to statistical properties. Linear dimensionality reduction techniques demonstrate the two basic statistical approaches. Principle Components Analysis (PCA) greedily isolates sources of variance in the data, while Independent Component Analysis (ICA) recovers a factorized representation, see~\cite{Ridgeway2016:FactorialRepr} for a recent review.

One important point to make is that {\bf disentanglement is sensitive to rotations of the latent embedding}. Following the example above, let us denote by $a$, $s$, and $h$, continuous values corresponding to face azimuth, skin brightness, and hair length. Then, if we change the ideal latent representation as follows
\begin{align}
\begin{pmatrix} a \\ s \\ h \end{pmatrix} \mapsto
   \begin{pmatrix*}[r]0.75 a + 0.25s + 0.61 h \\ 0.25a + 0.75s - 0.61 h \\ -0.61a + 0.61s + 0.50h \end{pmatrix*},
\end{align}
we obtain a representation that is equally expressive in terms of reconstruction (in fact we only multiplied with a 3D rotation matrix) but individual latent variables entirely lost their interpretable meaning.

\subsection{PCA and Latent Representations}

Let us examine more closely how PCA chooses the alignment of the latent embedding and why it matters.

It is well known \cite{autoassociation} that for a linear autoencoder with encoder $Y' \in \mathbb{R}^{d \times n}$, decoder $Y \in \mathbb{R}^{n \times d}$, and square error as reconstruction loss, the objective
\begin{align} \label{eq:pca_obj}
\min_{Y, Y'} \sum_{\xi \in X} \|\xi - YY'\xi \|^2
\end{align}
is minimized by the PCA decomposition. Specifically, by setting $Y' = P_d$, and $Y = P_d^\T$, for $P_d = \Unit_{d\times n}P \in \mathbb{R}^{d \times n}$,
where $P \in \mathbb{R}^{n \times n}$ is an orthogonal matrix formed by the $n$ normalized eigenvectors
(ordered by the magnitudes of the corresponding eigenvalues)
of the sample covariance matrix of $X$ and $\Unit_{d\times n} \in \mathbb{R}^{d \times n}$ is a trivial projection matrix.

However, there are many minimizers of \eqref{eq:pca_obj} that do not induce the same latent representation. In fact, it suffices to append $Y'$ with some invertible transformations (\eg rotations and scaling) and prefix $Y$ with their inverses. This geometrical intuition is well captured using the singular value decomposition (SVD), see also Figure \ref{fig:svd}.

\begin{figure*}
\begin{center}
\includegraphics[width=\textwidth]{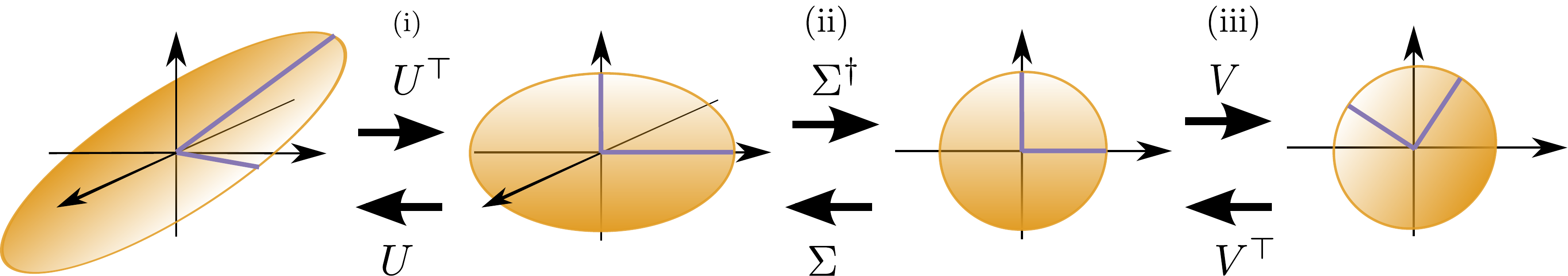}
\caption{Geometric interpretation of the singular value decomposition (SVD).
Sequential illustration of the effects of applying the corresponding SVD matrices of the encoder transformation $V\Sigma^\dagger U^\top$ (left to right) and decoder $U \Sigma V^\top$ (right to left).
We notice that steps (i) and (ii) of the encoder preserve the principle directions of the data. Step (iii), however, causes misalignment. In that regard, good encoders are the ones for which step (iii) is trivial. The same argument works for the decoder (in reverse order). This condition is equivalent (for non-degenerate transformations) to $U\Sigma V^\T$ having orthogonal columns (See Proposition \ref{prop:equi_decoder}, where this is phrased for the decoder).  \label{fig:svd}}
\end{center}
\end{figure*}

\begin{theorem}[SVD rephrased, \cite{Golub:1965:SVD}] Let $M \colon \mathbb{R}^n \to \mathbb{R}^d$ be a linear transformation (matrix). Then there exist
\begin{itemize}\setlength\itemsep{0em}
\item $U \colon \mathbb{R}^n \to \mathbb{R}^n$, an orthogonal transformation (matrix) of the input space,
\item $\Sigma \colon \mathbb{R}^n \to \mathbb{R}^d$ a ``scale-and-embed'' transformation (induced by a diagonal matrix),
\item $V \colon \mathbb{R}^d \to \mathbb{R}^d$, an orthogonal transformation (matrix) of the output space
\end{itemize}
such that $M = V\Sigma U^\T$.
\end{theorem}

\begin{remark}
For the sake of brevity, we will refer to orthogonal transformations (with slight abuse of terminology) simply as rotations.

\end{remark}

\begin{example}[Other minimizers of the PCA objective] \label{ex:rotation_matters}
Define $Y$ and $Y'$ with their SVDs as $Y = P^\T \Sigma Q$ and its pseudoinverse $Y' = Y^\dagger = Q^\T\Sigma^{\dagger} P$  and see that
\begin{align}
YY' = P^\T \Sigma QQ^\T\Sigma^{\dagger} P = P^\T \Unit_{d\times n}\Unit_{n\times d} P = P_d^\T P_d
\end{align}
so they are indeed also minimizers of the objective \eqref{eq:pca_obj} irrespective of our choice of $Q$ and $\Sigma$.

It is also straightforward to check that the only choices of $Q$, which respect the coordinate axes given by PCA, are for $|Q|$ to be a permutation matrix.
\end{example}

The take-away message (valid also in the non-linear case) from this example is:
\begin{center}
{\bf Different rotations of the same latent space are equally suitable for reconstruction.}
\end{center}
Following the PCA example, we formalize which linear mappings have the desired ``axes-preserving'' property.

\begin{prop}[Axes-preserving linear mappings]\label{prop:equi_decoder} Assume $M \in \mathbb{R}^{n \times d}$ with $d < n$ has $d$ distinct nonzero singular values. Then the following statements are equivalent:
\begin{enumerate}\setlength\itemsep{0em}
\item[(a)] The columns of $M$ are (pairwise) orthogonal.
\item[(b)] In every SVD of $M$ as $M = U\Sigma V^\T$, $|V|$ is a permutation matrix.
\end{enumerate}
\end{prop}

We strongly suggest developing a geometrical understanding for both cases (a) and (b) via Figure \ref{fig:svd}. For an intuitive understanding of the formal requirement of distinct eigenvalues, we refer to Supp.~\ref{sec:app:degeneracy}.

Take into consideration that once the encoder preserves the principle directions of the data, this already ensures an axis-aligned embedding. The same is true also if the decoder is axes-preserving, provided the reconstruction of the autoencoder is accurate.

\subsection{Related work}

Due to high activity surrounding VAEs, additional care is needed when it comes to evaluating novelty. To the best of our knowledge, two recent works address related questions and require special attention.

The authors of \cite{understanding-disent} also aim to explain good performance of ($\beta$--)VAE in disentanglement tasks. A compelling intuitive picture of the underlying dynamics is drawn and supporting empirical evidence is given. In particular, the authors {\it hypothesize} that ``$\beta$--VAE finds latent components which make different contributions to the log-likelihood term of the cost function [reconstruction loss]'', while suspecting that the diagonal posterior approximation is responsible for this behavior. Our theoretical analysis confirms both conjectures (see Section \ref{sec:walkthrough}).

Concurrent work \cite{anonymous2019isa-vae} develops ISA-VAE; another VAE-based architecture suited for disentanglement. Some parts of the motivation overlap with the content of our work. First, rotationally nonsymmetric priors are introduced for reasons similar to the content of Section \ref{sec:likelihood_invariant}. And second, both orthogonalization and alignment with PCA directions are empirically observed for VAEs applied to toy tasks.

\section{Results}
\label{sec:results}
\subsection{The problem with log-likelihood} \label{sec:likelihood_invariant}

The message from Example~\ref{ex:rotation_matters} and from the discussion about disentanglement is clear: latent space {\it rotation} matters. Let us look how the idealized objectives \eqref{eq:loglhood} and \eqref{eq:raw_elbo} handle this.

For a fixed rotation matrix $U$ we will be comparing a baseline encoder-decoder pair $(\Enc, \Dec)$ with a pair $(\Enc_{,U}, \Dec_{,U})$ defined as
\begin{align}
\Enc_{,U}(\mb{x}) &= U \Enc(\mb{x}), \\
\Dec_{,U}(\mb{z}) &= \Dec(U^\T\mb{z}).
\end{align}

The shortcomings of idealized losses are summarized in the following propositions.\newline

\begin{prop}[Log-likelihood rotation invariance] \label{prop:log_invariant} Let $\varphi$, $\theta$ be any choice of parameters for encoder-decoder pair $(\Enc_{,U}, \Dec_{,U})$.
Then, if the prior $p(\mb{z})$ is rotationally symmetric, the value of the log-likelihood objective \eqref{eq:loglhood} \emph{does not depend} on the choice of $U$.
\end{prop}

Note that the standard prior $\mathcal{N}(0, \mathcal{I})$ is rotationally symmetric.
This deficiency is not salvaged by the ELBO approximation.

\begin{prop}[ELBO rotation invariance] \label{prop:elbo_invariant} Let $\varphi$, $\theta$ be any choice of parameters for encoder-decoder pair $(\Enc_{,U}, \Dec_{,U})$.
Then, if the prior $p(\mb{z})$ is rotationally symmetric, the value of the ELBO objective \eqref{eq:raw_elbo} \emph{does not depend} on the choice of $U$.
\end{prop}

We do not claim novelty of these propositions, however we are not aware of their formalization in the literature.
The proofs can be found in Supplementary Material (Suppl.~\ref{sec:app:proofs}).
An important point now follows:
\begin{center}
{\bf Log-likelihood based methods (with rotationally symmetric priors) cannot claim to be designed to produce disentangled representations.}
\end{center}
However, \textbf{enforcing a diagonal posterior of the VAE encoder \eqref{eq:diagonal_enc} disrupts the rotational symmetry} and consequently the resulting objective \eqref{eq:obj_impl} escapes the invariance arguments.
Moreover, as we are about to see, this diagonalization comes with beneficial effects regarding disentanglement.
We assume this diagonalization was primarily introduced for different reasons (tractability, computational convenience), hence the ``by accident'' part of the title.

\subsection{Reformulating VAE loss}
\label{sec:reformulating_vae_loss}

The fact that VAEs were \emph{not meant} to promote orthogonality reflects in some technical challenges.
For one, we cannot follow a usual workflow of a theoretical argument; set up an idealized objective and find suitable approximations which allow for stochastic gradient descent (a top-down approach).
We need to do the exact opposite, start with the \emph{implemented loss function} and find the right simplifications that allow isolating the effects in question while preserving the original training dynamics (a bottom-up approach).
This is the main content of this section.

First, we formalize the typical situation in which VAE architectures ``shut down'' (fill with pure noise) a subset of latent variables and put high precision on the others.

\begin{definition}\label{def:polarized} We say that parameters $\varphi$, $\theta$ induce a \emph{polarized regime} if the latent coordinates $\{1, 2, \dots, d\}$ can be partitioned as $V_a \cup V_p$ (sets of active and passive variables) such that
\begin{enumerate}\setlength\itemsep{0em}
\item[(a)] $\mu^2_j(\mb{x}) \ll 1$ and $\sigma^2_j(\mb{x}) \approx 1$ for $j \in V_p$,
\item[(b)] $\sigma^2_j(\mb{x}) \ll 1$ for $j \in V_a$,
\item[(c)] The decoder ignores the passive latent components, \ie $$\frac{\partial \Dec(z)}{\partial z_j} = 0 \quad \forall j \in V_p.$$
\end{enumerate}
\end{definition}

The polarized regime simplifies the loss $L_{\text{KL}}$ from \eqref{eq:obj_impl}; part (a) ensures zero loss for passive variables and part (b) implies that
$\sigma^2_j(\mb{x}) \ll -\log(\sigma^2_j(\mb{x}))$. All in all, the per-sample-loss reduces to
\begin{align} \label{eq:kl_pol}
L_{\approx \text{KL}}(\xi) = \frac12 \sum_{j \in V_a} \left( \mu^2_j(\xi) - \log(\sigma^2_j(\xi)) -1 \right).
\end{align}
\textbf{We will assume the VAE operates in the polarized regime.} 
In Section \ref{sec:polarized_regime}, we show on multiple tasks and datasets that the two objectives align very early in the training. This behavior is well-known to practitioners.

Also, we approximate the reconstruction term in \eqref{eq:raw_elbo}, as it is most common, with a square loss
\begin{align}
\Lrec(\xi) = \E \|\Dec(\Enc(\xi)) - \xi \|^2
\end{align}
where the expectation is over the stochasticity of the encoder.
All in all, the loss we will analyze has the form
\begin{align} \label{eq:analyzed_loss}
\sum_{\xi \in X} \Lrec(\xi) + L_{\approx \text{KL}}(\xi).
\end{align}

Moreover, the reconstruction loss can be further decomposed into two parts; deterministic and stochastic. The former is defined by
\begin{align} \label{eq:det_loss}
\overline{L}_{\text{rec}}(\xi) = \| \Dec(\mu(\xi)) - \xi \|^2
\end{align}
and captures the square loss of the mean encoder. Whereas the stochastic loss
\begin{align}\label{eq:stoch_loss}
\hat{L}_{\text{rec}}(\xi) = \E \| \Dec(\mu(\xi)) - \Dec(\Enc(\xi)) \|^2
\end{align}
is purely induced by the noise injected in the encoder.

\begin{prop}\label{prop:loss_decomp} If the stochastic estimate $\Dec(\Enc(\xi))$ is unbiased around $\Dec(\mu(\xi))$, then
\begin{align}
\Lrec(\xi) = \overline{L}_{\text{rec}}(\xi) + \hat{L}_{\text{rec}}(\xi).
\end{align}
\end{prop}

This decomposition resembles the classical bias-variance decomposition of the square error \cite{James:2014:ISL:2517747}.

\subsection{The main result}\label{sec:main}

Now, we finally give theoretical evidence for the central claim of the paper:
\begin{center}
{\bf Optimizing the stochastic part of the reconstruction loss promotes local orthogonality of the decoder.}
\end{center}

On that account, we set up an optimization problem which allows us to optimize the stochastic loss \eqref{eq:stoch_loss} independently of the other two. This will isolate its effects on the training dynamics.

In order to make statements about local orthogonality, we introduce for each $\xi$ the Jacobian (linear approximation) $J_i$ of the decoder at point $\mu(\xi)$, \ie $$J_i = \frac{\partial \Dec(\mu(\xi))}{\partial \mu(\xi)}.$$
Since, according to \eqref{eq:diagonal_enc}, the encoder can be written as
$\Enc(\xi) = \mu(\xi) + \varepsilon(\xi)$ with
\begin{align}\label{eq:sampling}
\varepsilon(\xi) \sim\mathcal{N}\left(0,\diag{\sigma^2}(\xi)\right),
\end{align} we can approximate the stochastic loss \eqref{eq:stoch_loss} with
\begin{align}
  &\E_{\varepsilon(\xi)} \left\| \Dec(\mu(\xi)) - \left(\Dec(\mu(\xi)) + J_i \varepsilon(\xi)\right)\right\|^2 \nonumber\\
  &\qquad=\E_{\varepsilon(\xi)} \| J_i \varepsilon(\xi)\|^2, \label{eq:stochastic_loss_lin}
\end{align}

Although we aim to fix the deterministic loss \eqref{eq:det_loss}, we do not need to freeze the mean encoder and the decoder entirely. Following Example \ref{ex:rotation_matters}, for each $J_i$ and its SVD $J_i = U_i\Sigma_iV_i^\T$, we are free to modify $V_i$ as long we correspondingly (locally) modify the mean encoder.

Then we state the optimization problem as follows:
\begin{align} \label{eq:opt_problem}
\min_{V_i, \sigma^i_j > 0} \quad&\sum_{\xi \in X} \log  \E_{\varepsilon(\xi)}   \|J_i\varepsilon(\xi)\|^2  \\
\text{s.\,t.}\quad\quad
&\sum_{\xi \in X} L_{\approx \text{KL}}(\xi) = C, \label{eq:constraint}
\end{align}
where $\varepsilon(\xi)$ are sampled as in \eqref{eq:sampling}.

A few remarks are now in place.
\begin{itemize}\setlength\itemsep{0em}
\item This optimization is not over network parameters, rather directly over the values of all $V_i, \sigma^i_j$ (only constrained by \eqref{eq:constraint}).
\item Both the objective and the constraint concern \emph{global losses}, not per sample losses.
\item Indeed, none of $V_i, \sigma^i_j$ interfere with the rest of the VAE objective \eqref{eq:analyzed_loss}.
\end{itemize}
The presence of the (monotone) log function has one main advantage; we can describe {\bf all global minima} of \eqref{eq:opt_problem} in closed form. This is captured in the following theorem, the technical heart of this work.

\begin{theorem}[Main result]\label{thm:main} The following holds for optimization problem (\ref{eq:opt_problem}, \ref{eq:constraint}):
\begin{enumerate} \setlength\itemsep{0em}
\item[(a)] Every local minimum is a global minimum.
\item[(b)] In every global minimum, the columns of every $J_i$ are orthogonal.
\end{enumerate}
\end{theorem}

The full proof as well as an explicit description of the minima is given in Suppl.~\ref{sec:app:proof:main}.
However, an outline of the main steps is given in the next section on the example of a linear decoder.

The presence of the log term in \eqref{eq:opt_problem} admittedly makes our argument indirect.
There are, however, a couple of points to make. First, as was mentioned earlier, encouraging orthogonality was \emph{not a design feature} of the VAE.
In this sense, it is unsurprising that our results are also mildly indirect.

Also, and more importantly, the global optimality of Theorem \ref{thm:main} also implies that, locally, orthogonality is encouraged even for the pure (without logarithm) stochastic loss.

\begin{cor} For fixed $\xi \in X$ consider a subproblem of \eqref{eq:opt_problem} defined as
\begin{align} \label{eq:opt_problem_small}
\min_{V_i, \sigma^i_j > 0} \quad  &\E_{\varepsilon(\xi)}   \|J_i\varepsilon(\xi)\|^2  \\
\mathrm{s.\,t.}
\qquad\quad &L_{\approx \text{KL}}(\xi) = C_i. \label{eq:constraint_sub}
\end{align}
Also then, the result on the structure of local (global) minima holds:
\begin{enumerate}\setlength\itemsep{0em}
\item[(a)] Every local minimum is a global minimum.
\item[(b)] In every global minimum, the columns of every $J_i$ are orthogonal.
\end{enumerate}

\end{cor}

All in all, Theorem \ref{thm:main} justifies the central message of the paper stated at the beginning of this section. The analogy with PCA is now also clearer. Locally, VAEs optimize a tradeoff between reconstruction and orthogonality.

This result is unaffected by the potential $\beta$ term in Equation~\eqref{eq:raw_elbo}, although an
 appropriate $\beta$ might be required to ensure the polarized regime.

\section{Proof outline}
\label{sec:walkthrough}

In this section, we sketch the key steps in the proof of Theorem \ref{thm:main} and, more notably, the intuition behind them. The full proof can be found in Suppl.~\ref{sec:app:proof:main}.

We will restrict ourselves to a simplified setting.
Consider a linear decoder $M$ with SVD $M = U\Sigma V^T$, which removes the necessity of local linearization.
This reduces the objective \eqref{eq:opt_problem} from a ``global'' problem over all examples $\xi$
to an objective where we have the same subproblem for each $\xi$.

As in optimization problem (\ref{eq:opt_problem}, \ref{eq:constraint}), we resort to fixing the mean encoder (imagine a well performing one).

In the next paragraphs, we separately perform the optimization over the parameters ${\bm \sigma}$ and the optimization over the matrix $V$.

\subsection{Weighting precision}

For this part, we fix the decoder matrix $M$ and optimize over values $\bm{\sigma}^2 = (\sigma^2_1, \dots, \sigma^2_d)$.
The simplified objective is
\begin{align}
\min_{\bm{\sigma}} \quad  &\E_{\varepsilon \sim \mathcal{N}(0,\diag(\mb{\sigma}^2))}   \|M\varepsilon\|^2 \label{eq:simpl_obj} \\
\text{s.\,t.} 
\qquad\quad& \sum_{j} -\log \sigma^2_j = C, \label{eq:simpl_cons}
\end{align}
where the $\| \mu \|^2$ terms from \eqref{eq:kl_pol} disappear since the mean encoder is fixed.

The values $-\log(\sigma_j)$ can now be thought of as precisions allowed for different latent coordinates.
The $\log$ functions even suggests thinking of the number of significant digits.
Problem \eqref{eq:simpl_obj} then asks to distribute the ``total precision budget`` so that the deviation from decoding ``uncorrupted'' values is minimal.

We will now solve this problem on an example linear decoder $M_1 \colon \mathbb{R}^2 \to \mathbb{R}^3$ given by
\begin{align}
M_1 \colon \begin{pmatrix} x \\ y \end{pmatrix} \mapsto \begin{pmatrix*}[r] 4x+y \\ -3x+y \\ 5x-y \end{pmatrix*}.
\end{align}
Already here we see, that the latent variable $x$ seems more influential for the reconstruction. We would expect that $x$ receives higher precision than $y$.

Now, for $\bm{\varepsilon} = (\varepsilon_x, \varepsilon_y)$, we compute
\begin{align*}
\|M_1\bm{\varepsilon}\|^2 = \|4\varepsilon_x + \varepsilon_y\|^2 + \|{-3}\varepsilon_x + \varepsilon_y\|^2 + \|5\varepsilon_x - \varepsilon_y\|^2
\end{align*}
and after taking the expectation, we can use the fact that $\bm{\varepsilon}$ has zero mean and write
\begin{align*}
\E \|&M_1\bm{\varepsilon}\|^2 = \\
&\var (4\varepsilon_x + \varepsilon_y) + \var ({-3}\varepsilon_x + \varepsilon_y) + \var(5\varepsilon_x - \varepsilon_y).
\end{align*}
Finally, we use that for uncorrelated random variables $A$ and $B$ we have $\var(A + cB) = \var A + c^2 \var B$. After rearranging we obtain
\begin{align*}
\E \|M_1\bm{\varepsilon}\|^2 &= \sigma^2_x(4^2\! + ({-3})^2 + 5^2)\! +\! \sigma^2_y(1^2 + 1^2\! + (-1)^2) \\
&= 50\sigma^2_x + 3\sigma^2_y,
\end{align*}
where $\bm{\sigma} = (\sigma^2_x, \sigma^2_y)$.
Note that the coefficients are the {\bf squared norms of the column vectors} of $M_1$.

This turns the optimization problem \eqref{eq:simpl_obj} into a simple exercise, particularly after realizing that \eqref{eq:simpl_cons} fixes the value of the product $\sigma_x\sigma_y$.
Indeed, we can even set $a^2 = 50\sigma_x$ and $b^2 = 3\sigma_y$ in the trivial inequality $a^2 + b^2 \geq 2ab$ and find that
\begin{align}
\E \|M_1\varepsilon&\|^2 = 50\sigma^2_x + 3\sigma^2_y \geq 2 \cdot \sqrt{50 \cdot 3}\cdot e^{-C} \approx 24.5e^{-C}, \label{eq:simple_ag}
\end{align}
with equality achieved when $\sigma^2_x / \sigma^2_y = 3/50$.
This also implies that the precision ${-}\log \sigma^2_x$ on variable $x$ will be considerably higher than for $y$, just as expected.

Two remarks regarding the general case follow.
\begin{itemize}\setlength\itemsep{0em}
\item The full version of inequality \eqref{eq:simple_ag} relies on the concavity of the $\log$ function; in particular, on (a version of) Jensen's inequality.
\item The minimum value of the objective depends on the product of the column norms. This also carries over to the unsimplified setting.
\end{itemize}

\subsection{Isolating sources of variance}
\label{sec:isolatingsourcesofvariance}
Now that we can find optimal values of precision, the focus changes on optimally rotating the latent space.
In order to understand how such rotations influence the minimum of objective \eqref{eq:simpl_obj}, let us consider the following example in which we again resort to decoder matrix $M_2 \colon \mathbb{R}^2 \to \mathbb{R}^3$.

Imagine, the encoder alters the latent representation by a $45^\circ$ rotation.
Then we can adjust the decoder $M_1$ by first undoing this rotation.
In particular, we set $M_2 = M_1R_{45^\circ}^\T$, where $R_\theta$ is a 2D rotation matrix, rotating by angle $\theta$.
We have
$$M_2 \colon \begin{pmatrix} x' \\ y' \end{pmatrix} \mapsto \begin{pmatrix*}[r] \frac12 \sqrt{2}(3x'+5y') \\ \sqrt{2}({-}2x'-y') \\ \sqrt{2}(3x'+2y') \end{pmatrix*}$$
and performing analogous optimization as before gives
\begin{align}
\E \|M_2\bm{\varepsilon}&\|^2\!=\!\frac{61}{2}\sigma^2_x\!+\!\frac{45}{2} \sigma^2_y \geq 2 \sqrt{\frac{61 \cdot 45}{4}} e^{-C}\!\approx\!52.4e^{-C}. \label{eq:simple_ag_worse}
\end{align}

We see that the minimal value of the objective is more than twice as high, a substantial difference.
On a high level, the reason $M_1$ was a better choice of a decoder is that the variables $x$ and $y$ had very different impact on the reconstruction. This allowed to save some precision on variable $y$, as it had smaller effect, and use it on $x$, where it is more beneficial.

For a higher number of latent variables, one way to achieve a ``maximum stretch'' among the impacts of latent variables, is to pick them greedily, always picking the next one so that its impact is maximized. This is, at heart, the greedy algorithm for PCA.

Let us consider a slightly more technical statement.
We saw in \eqref{eq:simple_ag} and \eqref{eq:simple_ag_worse} that after finding optimal values of $\bm{\sigma}$ the remaining objective is the product of the column norms of matrix $M$. Let us denote such quantity by $\col(M) = \prod_j \|M_{\cdot j}\|$. Then for a fixed matrix $M$, we optimize
\begin{align}
\min_V \col(MV^\T)
\end{align}
over orthogonal matrices $V$.

\begin{figure}
  \centering
  \includegraphics[width=.8\linewidth]{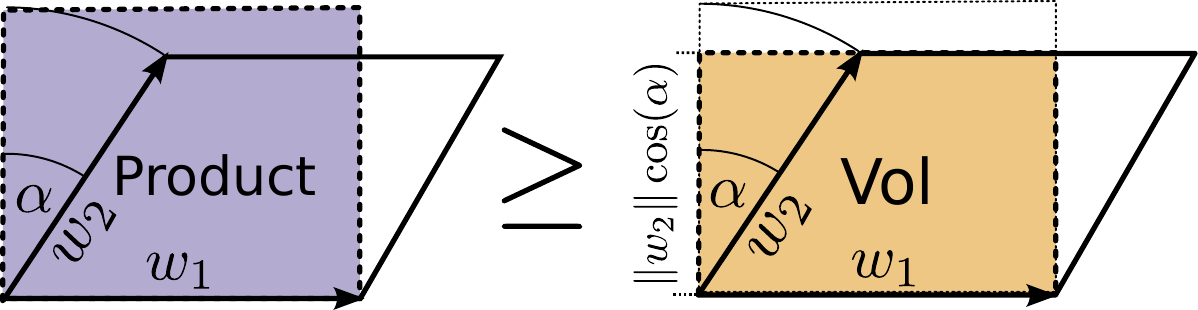}
  \caption{2D illustration of orthogonality in $MV^{\T}$. The vectors $w_1,w_2$ are the columns of $MV^\T$.
    Minimizing the product $\|w_1\|\|w_2\|$ while maintaining the volume $\|w_1\|\|w_2\|\cos(\alpha)$ results in $w_1 \perp w_2$.
  }
  \label{fig:ortho-vol}
\end{figure}

This problem can be interpreted geometrically.
The column vectors of $MV^\T$ are the images of base vectors $e_j$.
Consequently, the product gives an upper bound on the volume (the image of the unit cube)
\begin{align}
\prod_j\|MV^\T e_j\| \ge \mathrm{Vol}(\{ MV^\T x \colon x \in [0, 1]^d\})\,.\label{eqn:vol-det}
\end{align}
However, as orthogonal matrices $V$ are isometries, they do not change this volume.
Also, the bound \eqref{eqn:vol-det} is tight precisely when the vectors $MV^\T e_j$ are orthogonal.
Hence, the only way to optimize $\col(MV^\T)$ is by tightening the bound, that is by finding $V$ for which the column vectors of $MV^\T$ are orthogonal, see Figure \ref{fig:ortho-vol} for an illustration.
In this regards, it is important that $M$ performs a different scaling along each of the axis (using $\Sigma$),
 which allows for changing the angles among the vectors $MV^\T e_j$ (cf.\! Figure \ref{fig:svd}).

\begin{table*}\centering
  \caption{Results for the distance to orthogonality DtO of the decoder (Equation \ref{eq:dto}) and disentanglement score for different architectures and datasets. Lower DtO values are better and higher Disent.\ values are better. Random decoders provide a simple baseline for the numbers.}
  \begin{tabular}{|l|l@{\,}c @{\,}c|cccc|c|}
\hline
& & & & \textbf{$\beta$-VAE} & \textbf{VAE} & \textbf{AE} & \textbf{$\beta$-VAE$_\Sigma$} & \textbf{Random Decoder} \\\hline\hline
\textbf{dSprites} & \textbf{Disent.} & & $\uparrow$ & $\bf 0.33 \pm 0.15$ & $0.21 \pm 0.10$ & $0.09 \pm 0.04$ & $0.12 \pm 0.06$ & \\
 & \textbf{DtO} & $\downarrow$ & & $\bf 0.76 \pm 0.08$ & $1.08 \pm 0.15$ & $1.62 \pm 0.03$ & $1.73 \pm 0.14$ & $1.86 \pm 0.11$\\
\textbf{Synth.\ Lin. } & \textbf{Disent.} & & $\uparrow$  & $\bf 0.99 \pm 0.01$ & -- & $0.71 \pm 0.19$ & $0.71 \pm 0.31$ & \\
 & \textbf{DtO} & $\downarrow$  & & $\bf 0.00 \pm 0.00$ & -- & $0.33 \pm 0.18$ & $0.34 \pm 0.35$ & $0.79 \pm 0.21$\\
\textbf{Synth.\ Non-Lin.} & \textbf{Disent.} & & $\uparrow$  & $\bf 0.73 \pm 0.16$ & -- & $0.59 \pm 0.30$ & $0.42 \pm 0.24$ & \\
 & \textbf{DtO} & $\downarrow$ & & $\bf 0.18 \pm 0.02$ & -- & $0.54 \pm 0.13$ & $0.55 \pm 0.02$ & $0.89 \pm 0.16$\\
\textbf{MNIST} & \textbf{DtO} & $\downarrow$ & & -- & $\bf 1.59 \pm 0.08$ & $1.83 \pm 0.05$ & $1.93 \pm 0.08$ & $2.11 \pm 0.11$\\
\textbf{fMNIST} & \textbf{DtO} & $\downarrow$ & & -- & $\bf 1.36 \pm 0.05$ & $1.87 \pm 0.03$ & $2.02 \pm 0.08$ & $2.11 \pm 0.11$\\
 \hline
\end{tabular}
\label{tbl:l2}
\end{table*}
\section{Experiments}
\label{sec:experiments}
We performed several experiments with different architectures and datasets to validate
our results empirically. We show the prevalence of the polarized regime, the strong orthogonal effects of the ($\beta$-)VAE, as well as the links to disentanglement.

\subsection{Setup}
\newcommand{\ourparagraph}[1]{\vspace*{0.2\baselineskip}\noindent\textbf{#1}}
\ourparagraph{Architectures.} We evaluate the classical VAE, $\beta$-VAE, a plain autoencoder, and $\beta$-$\text{VAE}_{\Sigma}$, where the latter removes the critical diagonal approximation \eqref{eq:diagonal_enc} and produces a full covariance matrix $\Sigma(\xi)$ for every sample. The resulting KL term of the loss is changed accordingly (see Suppl.~\ref{sec:app:fullVAE} for details).

\ourparagraph{Datasets.} We evaluate on the well-known datasets dSprites \cite{dsprites17}, MNIST \cite{Lecun1998:MNIST} and FashionMNIST \cite{fashion_MNIST}, as well as on two synthetic ones.
For both synthetic tasks the input data $X$ is generated by embedding a unit square $V = [0, 1]^2$ into a higher dimension. The latent representation is then expected to be disentangled with respect to axes of $V$. In one case (\emph{Synth.\ Lin.}) we used a linear transformation $f_\mathrm{lin} \colon \mathbb{R}^2 \rightarrow \mathbb{R}^3$ and in the other one a non-linear (\emph{Synth.\ Non-Lin.}) embedding $f_\mathrm{non-lin} \colon \mathbb{R}^2 \rightarrow \mathbb{R}^{6}$. The exact choice of transformations can be found in Suppl.~\ref{sec:app:exp}.
Further information regarding network structures and training parameters is also provided in Suppl.~\ref{sec:app:training}.

\ourparagraph{Disentanglement metric.} For quantifying the disentanglement of a representation, the so called Mutual Information Gap (MIG) was introduced in \cite{tc-beta-vae}. As MIG is not well defined for continuous variables, we use an adjusted definition comprising both continuous and discrete variables, simply referred to as \emph{Disentanglement score}. Details are described in Suppl.~\ref{sec:app:mig}. Just as in the case of MIG, the Disentanglement score is a number between $0$ and $1$, where higher value means stronger disentanglement.

\ourparagraph{Orthogonality metric.} For measuring the practical effects of Theorem \ref{thm:main}, we introduce a measure of non-orthogonality. As argued in Proposition \ref{prop:equi_decoder} and Figure \ref{fig:svd}, for a good decoder $M$ and its SVD $M = U\Sigma V^\T$, the matrix $V$ should be trivial (a signed permutation matrix). We measure the non-triviality with the \emph{Distance to Orthogonality} (DtO) defined as follows. For each $\xi$, $i=1,\dots,N$, employing again the Jacobian $J_i$ of the decoder at $\xi$ and its SVD $J_i = U_i\Sigma_iV_i^\T$ and define
\begin{align}
\text{DtO} = \frac{1}{N} \sum_{i=1}^N \| V_i - P(V_i) \|_F,\label{eq:dto}
\end{align}
where $\|\cdot \|_F$ is the Frobenius norm and $P(V_i)$ is a signed permutation matrix that is closest to $V$ (in $L^1$ sense).
Finding the nearest permutation matrix is solved to optimality via mixed-integer linear programming (see Suppl.~\ref{sec:app:procrutes}).
\begin{table}\centering
  \caption{Percentage of training time where $\Delta_{KL} < 3\,\%$ (Eq.~\eqref{eq:relative_kl_error}) continuously until the end.
    Reported for $\beta$-VAE with low (dataset dependent) and high ($10$) latent dimension. \label{tbl:polarized_regime}}
\begin{tabular}{|l|cc|}
\hline
 & \textbf{$\beta$-VAE} (dep.) & \textbf{$\beta$-VAE} (10) \\\hline\hline
\textbf{ dSprites } & $97.8\,\%$ & $90.6\,\%$\\
\textbf{ fMNIST } & $99.8\,\%$ & $97.7\,\%$\\
\textbf{ MNIST } & $99.8\,\%$ & $99.5\,\%$\\
\textbf{ Synth. Lin. } & $99.8\,\%$ & $96.7\,\%$\\
\textbf{ Synth. Non-Lin. } & $99.9\,\%$ & $98.5\,\%$\\
\hline
\end{tabular}
\end{table}
\subsection{Polarized regime}
\label{sec:polarized_regime}
In Section \ref{sec:reformulating_vae_loss}, we assumed VAEs operate in a polarized regime and approximated $L_{\text{KL}}$, the KL term of the implemented objective \eqref{eq:obj_impl}, with $L_{\approx \text{KL}}$ \eqref{eq:kl_pol}.
In Table \ref{tbl:polarized_regime}
we show that the polarized regime is indeed dominating the training in all examples after a short initial phase.
We report the fraction of the training time in which the relative error
\begin{align}
 \Delta_{KL} = \frac{|L_{\text{KL}} - L_{\approx \text{KL}}|}{L_{\text{KL}}}
\label{eq:relative_kl_error}
\end{align}
stays below $3\,\%$ continuously until the end (evaluated every 500 batches). Active variables can be selected by $\sqrt{\mathrm{var}\left( \mu_j\left( \xi \right) \right)} > 0.5$.

\subsection{Orthogonality and Disentanglement}
Now, we provide evidence for Theorem~\ref{thm:main}
by investigating the DtO~\eqref{eq:dto} for a variety of architectures and datasets,
 see Table~\ref{tbl:l2}.
The results clearly support the claim that the VAE based architectures indeed strive for local orthogonality.
By generalizing the $\beta$-VAE architecture, such that the approximate posterior is any multivariate Gaussian ($\beta$-VAE$_{\Sigma}$), the objective becomes rotationally symmetric (just as the idealized objective).
As such, no specific alignment is prioritized.
The simple autoencoders also do not favor particular orientations of the latent space.

Another important observation is the clear correlation between DtO and the disentanglement score.
We show this in Figure~\ref{fig:mig_l2} where different restarts of the same $\beta$-VAE architecture on the dSprites dataset are displayed. We used the state-of-the-art value $\beta=4$ \cite{beta-vae}.
Additional experiments are reported in Suppl.~\ref{sec:app:addexp}.
 \begin{figure}[h!]\centering
  \includegraphics[width=0.9\columnwidth]{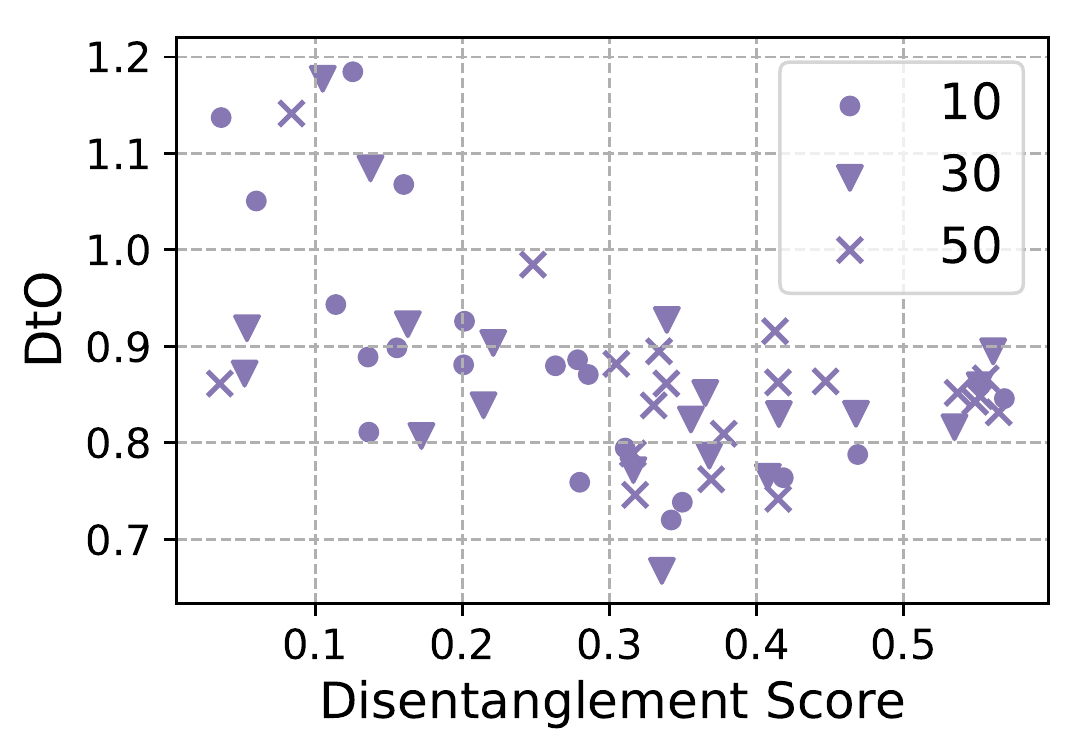}
\caption{Alignment of the latent representation (low DtO, \eqref{eq:dto}) results in better disentanglement (higher score). Each datapoint corresponds to an independent run with $10$, $30$, or $50$ epochs.}
\label{fig:mig_l2}
\end{figure}

\section{Discussion}

We isolated the mechanism of VAE that leads to local orthogonalization and, in effect, to performing local PCA. Additionally, we demonstrated the functionality of this mechanism in intuitive terms, in formal terms, and also in experiments. We also explained why this behavior is desirable for enforcing disentangled representations.

Our insights show that VAEs make use of the differences in variance to form the representation in the latent space -- collapsing to PCA in the linear case.
This does not {\it directly} encourage factorized latent representations.
 With this in mind, it makes perfect sense that recent improvements of ($\beta$-)VAE \cite{tc-beta-vae, factor-vae, anonymous2019isa-vae} incorporate additional terms promoting precisely independence.

It is also unsatisfying that VAEs promote orthogonality somewhat indirectly. It would seem that designing architectures allowing explicit control over this feature would be beneficial.

\section*{Acknowledgements}
We thank the whole Autonomous Learning Group at MPI IS, as well as Friedrich Solowjow for the fruitful and invaluable discussions. Also, we thank the International Max Planck Research School for Intelligent Systems (IMPRS-IS) for supporting Dominik Zietlow.

\appendix

\twocolumn[%
    \begin{center}
      {\Large \bf Supplementary Material \vspace*{1em}}
    \end{center}
    ]

The supplementary information is structured as follows. We start with a remark on Table~\ref{tbl:polarized_regime} and then provide the proofs in Section~\ref{sec:app:proof:main}. Section~\ref{sec:app:exp} reports the details of the experiments followed by additional experiments in Section \ref{sec:app:addexp}.

\subsection*{Remark on Table \ref{tbl:polarized_regime}}
Some dataset-architecture combinations listed in Table \ref{tbl:l2} are omitted for the following reasons.

On the one hand, calculating the Disentanglement Score for MNIST and fMNIST does not make sense,
 as the generating factors are not given (the one categorical label cannot serve as replacement).
Consequently, as the values of $\beta$ are chosen according to this score, we do not report $\beta$-VAE numbers for these datasets.
On the other hand, for either synthetic task, the regular VAE vastly overprunes,
 see Figure~\ref{fig:app:sweep}, and the values become meaningless.

\section{Proofs}\label{sec:app:proofs}

\subsection{Proof of Theorem \ref{thm:main}}\label{sec:app:proof:main}

\paragraph{Proof strategy:} For part (b), we aim to derive a lower bound on the objective \eqref{eq:opt_problem}, that is independent from the optimization variables $\sigma^2_j(\xi)$ and $V_i$. Moreover, we show that this lower bound is tight for some specific choices of $\sigma^2_j(\xi)$ and $V_i$, \ie the global optima. For these choices, all $J_i$ will have orthogonal columns.

The strategy for part (a) is to show that whenever $\sigma^2_j(\xi)$ and $V_i$ do not induce a global optimum, we can find a small perturbation that decreases the objective function. Thereby showing that local minima do not exist.

\paragraph{Technical lemmas:}
We begin with introducing a few useful statements.
First is the inequality between arithmetic and geometric mean; a consequence of Jensen's inequality.

\begin{lemma}[AM-GM inequality] \label{lemma:amgm} Let $a_1$, \dots, $a_N$ be nonnegative real numbers. Then
\begin{align}
\frac{1}{N} \sum_{i=1}^{N} a_i \geq \left(\prod_{i=1}^{N} a_i \right)^{1/N}
\end{align}
with equality occuring if and only if $a_1 = a_2 = \cdots = a_n$.
\end{lemma}

The second bound to be used is the classical Hadamard's inequality.

\begin{lemma}[Hadamard's inequality \cite{trefethen1997numerical}]\label{lemma:hadamard} Let $M \in \mathbb{R}^{k \times k}$ be non-singular matrix with column vectors $c_1$, \dots, $c_k$. Then
\begin{align}
\prod_{i=1}^{k} \|c_i\| \geq |\det M|
\end{align}
with equality if and only if the vectors $c_1$, \dots, $c_k$ are pairwise orthogonal.
\end{lemma}

And finally a simple lemma for characterizing matrices with orthogonal columns.

\begin{lemma}[Column orthogonality]\label{lemma:ortho} Let $M \in \mathbb{R}^{n \times d}$ be a matrix and let $M = U \Sigma V^\T$ be its singular value decomposition. Then the following statements are equivalent:
\begin{enumerate}\setlength\itemsep{0em}
\item[(a)] The columns of $M$ are (pairwise) orthogonal.
\item[(b)] The matrix $M^\T M$ is diagonal.
\item[(c)] The columns of $\Sigma V^\T$ are (pairwise) orthogonal.
\end{enumerate}
\end{lemma}
\begin{proof} The equivalence of (a) and (b) is immediate. For equivalence of (a) and (c) it suffices to notice that if we set $M' = \Sigma V^\T$, then
\begin{align}
M'^\T M' = V \Sigma^\T \Sigma V^\T = M^\T M.
\end{align}
The equivalence of (a) and (b) now implies that $M$ has orthogonal columns if and only if $M'$ does.
\end{proof}

\paragraph{Initial considerations:}

First, without loss of generality, we will ignore all passive latent variables (in the sense of Definition \ref{def:polarized}). Formally speaking, we will restrict to the case when the local decoder mappings $J_i$ are non-degenerate (\ie have non-zero singular values). Now $d$ denotes the dimensionality of the latent space with $d=|V_a|$.

\medskip

Next, we simplify the loss $L_{\approx \text{KL}}$, Equation \ref{eq:kl_pol}. Up to additive and multiplicative constants, this loss can be, for a fixed sample $\xi \in X$, written as
\begin{align}
\|\mu(\xi)\|^2 + \sum_{j=1}^{d} {-}\log(\sigma^2_j(\xi)).
\end{align}
In the optimization problem (\ref{eq:opt_problem}, \ref{eq:constraint}) the values $\mu(\xi)$ can only be affected via applying an orthogonal transformation $V_i$.
But such transformation are norm-preserving (isometric) and hence the values $\|\mu(\xi)\|^2$ do not change in the optimization.
As a result, we can restate the constraint \eqref{eq:constraint} as
\begin{align} \label{eq:new_constraint}
\sum_{\xi \in X} \sum_{j=1}^{d} {-}\log(\sigma^2_j(\xi)) = C_1
\end{align}
for some constant $C_1$.

\paragraph{Proof of Theorem \ref{thm:main}(b):}

Here, we explain how Theorem \ref{thm:main}(b) follows from the following two propositions.

\begin{prop} \label{prop:ineq1} For a fixed sample $\xi \in X$ let us denote by $c_1$, \dots, $c_d$ the column vectors of $J_i$. Then
\begin{align}
\Eeps \|J_i\bm{\varepsilon}(\xi)\|^2 \geq d \left( \prod_{j=1}^d \|c_j\|^2 \sigma^2_j(\xi) \right)^{1/d}
\end{align}
with equality if and only if $\|c_j\|^2 \sigma^2_j(\xi) = \|c_k\|^2 \sigma^2_k(\xi)$ for every $j, k \in \{1, \dots, d\}$.
\end{prop}

\begin{prop} \label{prop:ineq2} Let $M \in \mathbb{R}^{n \times d}$, where $d < n$, be a matrix with column vectors $c_1$, \dots, $c_d$ and nonzero singular values $s_1$, \dots, $s_d$. Then
\begin{align}
\prod_{j=1}^d \|c_j\| \geq \psdet(M),
\end{align}
where by $\psdet(M)$ we denote the product of the singular values of $M$.
Equality occurs if and only if $c_1$, \dots, $c_d$ are pairwise orthogonal.
\end{prop}

First, Proposition \ref{prop:ineq2} allows making further estimates in the inequality from Proposition \ref{prop:ineq1}. Indeed, we get
\begin{align}
\Eeps \|J_i\bm{\varepsilon}(\xi)\|^2 \geq d \left( \left(\psdet(J_i)\right)^2  \prod_{j=1}^d \sigma^2_j(\xi) \right)^{1/d}
\end{align}
and after applying the (monotonous) log function we are left with
\begin{align}
\log \Eeps &\|J_i\bm{\varepsilon}(\xi)\|^2 \geq \\
&\log(d) + \frac{2}{d} \log(\psdet(J_i)) + \frac{1}{d} \sum_{j=1}^d \log(\sigma^2_j(\xi)).
\end{align}
Finally, we sum over the samples $\xi \in X$ and simplify via \eqref{eq:new_constraint} as
\begin{align}
\sum_{\xi \in X} &\log \Eeps \|J_i\bm{\varepsilon}(\xi)\|^2 \geq \nonumber \\
&N\log(d) - \frac{C_1}{d} + \frac{2}{d} \sum_{\xi \in X} \log(\psdet(J_i)) \label{ineq:final}.
\end{align}
The right-hand side of this inequality is independent from the values of $\sigma^2_j(\xi)$, as well as from the orthogonal matrices $V_i$, since these do not influence the singular values of any $J_i$.

Moreover, it is possible to make inequality \eqref{ineq:final} tight (\ie reach the global minimum), by setting $\sigma^2_j(\xi)$ as hinted by Proposition \ref{prop:ineq1} and by choosing the matrices $V_i$ such that every $J_i$ has orthogonal columns (this is clearly possible as seen in Proposition \ref{prop:equi_decoder}).

This yields the desired description of the global minima of \eqref{eq:opt_problem}. \qed

\paragraph{Proof of Proposition \ref{prop:ineq1}:}
We further denote by $r_1$, \dots, $r_n$ the row vectors of $J_i$, and by $a_{r,c}$ the element of $J_i$ at $r$-th row and $c$-th column. With sampling $\bm{\varepsilon}(\xi)$ according to
\begin{align}
\varepsilon(\xi) \sim\mathcal{N}\left(0,\diag{\sigma^2(\xi)}\right),
\end{align}
we begin simplifying the objective \eqref{eq:opt_problem} with
\begin{align}
 \Eeps \|J_i\bm{\varepsilon}(\xi)\|^2 &= \Eeps \sum_{k=1}^n \|r_k^\T \bm{\varepsilon}(\xi) \|^2 \\
 &= \sum_{k=1}^n \Eeps \|r_k^\T \bm{\varepsilon}(\xi) \|^2.
\end{align}
Now, as the samples $\bm{\varepsilon}(\xi)$ are zero mean, we can further write
\begin{align}
\sum_{k=1}^n \Eeps \|r_k^\T \bm{\varepsilon}(\xi) \|^2 = \sum_{k=1}^n \var (r_k^\T \bm{\varepsilon}(\xi)).
\end{align}
Now we use the fact that for uncorrelated random variables $A$ and $B$ we have $\var(A + cB) = \var A + c^2 \var B$. This allows to expand the variance of the inner product as
\begin{align}
\var (r_k^\T \bm{\varepsilon}(\xi)) &= \var \left(\sum_{j=1}^d a_{k, j} \bm{\varepsilon}_j(\xi)\right) \\
&= \sum_{j=1}^d  a^2_{k, j} \var \bm{\varepsilon}_j(\xi) = \sum_{j=1}^d  a^2_{k, j} \sigma^2_j(\xi).\nonumber
\end{align}

Now, we can regroup the terms via
\begin{align}
\sum_{k=1}^n \var (r_k^\T \bm{\varepsilon}(\xi)) &= \sum_{k=1}^n \sum_{j=1}^d  a^2_{k, j} \sigma^2_j(\xi) \nonumber \\
&= \sum_{j=1}^d \sum_{k=1}^n a^2_{k, j} \sigma^2_j(\xi)  \nonumber \\
&= \sum_{j=1}^d \|c_j\|^2 \sigma^2_j(\xi).
\end{align}
All in all, we obtain
\begin{align}
\Eeps \|J_i\bm{\varepsilon}(\xi)\|^2 = \sum_{j=1}^d \|c_j\|^2 \sigma^2_j(\xi).
\end{align}
from which the desired inequality follows via setting $a_j = \|c_j\|^2 \sigma^2_j(\xi)$ for $j = 1$, \dots, $d$ in Lemma \ref{lemma:amgm}. Indeed, then we have
\begin{align} \label{eq:am-gm-used}
\sum_{j=1}^d \|c_j\|^2 \sigma^2_j(\xi) \geq d \left( \prod_{j=1}^d \|c_j\|^2 \sigma^2_j(\xi) \right)^{1/d}
\end{align}
as required.
\qed

\paragraph{Proof of Proposition \ref{prop:ineq2}:}

As the first step, we show that both sides of the desired inequality are invariant to multiplying the matrix $M$ from the left with an orthogonal matrix $U \in \mathbb{R}^{n \times n}$.

For the right-hand side, this is clear as the singular values of $UM$ are identical to those of $M$. As for the left-hand side, we first need to realize that the vectors $c_j$ are the images of the canonical basis vectors $e_j$, \ie $c_j = Me_j$ for $j = 1, \dots, d$. But since $U$ is an isometry, we have $\|UMe_j\| = \|Me_j\| = \|c_j\|$ for every $j$, and hence also the column norms are intact by prepending $U$ to $M$.

This allows us to restrict to matrices $M$ for which the SVD has a simplified form $M = \Sigma V^\T$. Next, let us denote by $\Sigma_{d \times d}$ the $d \times d$ top-left submatrix of $\Sigma$. Note that $\Sigma_{d \times d}$ contains all nonzero elements of $\Sigma$. As a result, the matrix $M' = \Sigma_{d \times d} V^\T$ contains precisely the nonzero rows of the matrix $M$. This implies
\begin{align} \label{eq:equi_ortho}
M^\T M = M'^\T M'.
\end{align}
In particular, the column vectors $c'_j$ of $M'$ have the same norms as those of $M$.
Now we can write
\begin{align} \label{eq:hadam-used}
\prod_{j=1}^d \|c_j\| = \prod_{j=1}^d \|c'_j\| \geq |\det(M')| = \psdet(M),
\end{align}
where the inequality follows from Lemma \ref{lemma:hadamard} applied to nonsingular matrix $M'$. Equality in Lemma \ref{lemma:hadamard} occurs precisely if the columns of $M'$ are orthogonal. However, according to Lemma \ref{lemma:ortho} and \eqref{eq:equi_ortho}, it also follows that the columns of $M'$ are orthogonal if and only if the columns of $M$ are. Note that Lemma \ref{lemma:ortho}(c) is needed for covering the reduction performed in the first two paragraphs. \qed

\paragraph{Proof of Theorem \ref{thm:main}(a):}

We show the nonexistence of local minima as follows. For any values of $\sigma^2_j(\xi)$ and $V_i$ that do not minimize the objective function \eqref{eq:opt_problem}, we find a small perturbation that improves this objective.

All estimates involved in establishing inequality \eqref{ineq:final} rely on either Lemma \ref{lemma:amgm} or Lemma \ref{lemma:hadamard}, where in both cases, the right-hand side was kept fixed. We show that both of these inequalities can be tightened in such fashion by small perturbations in their parameters.

\begin{lemma}[Locally improving AM-GM] For any non-negative values $a_1$, \dots, $a_N$ for which
\begin{align} \label{equi:am-gm-not-eq}
\frac{1}{N} \sum_{i=1}^{N} a_i > \left(\prod_{i=1}^{N} a_i \right)^{1/N}
\end{align}
there exists a small perturbation $a'_i$ of $a_i$ for $i = 1, \dots, N$ such that
\begin{align}
\frac{1}{N} \sum_{i=1}^{N} a_i > \frac{1}{N}& \sum_{i=1}^{N} a'_i \geq \label{eq:ag-lemma1} \\
&\left(\prod_{i=1}^{N} a'_i \right)^{1/N} = \left(\prod_{i=1}^{N} a_i \right)^{1/N}  \label{eq:ag-lemma2}
\end{align}
\end{lemma}
\begin{proof}
  Since \eqref{equi:am-gm-not-eq} is a sharp inequality, we have $a_i > a_j$ for some $i \neq j$. Then setting $a'_i = a_i / (1+\delta)$, $a'_j = a_j(1+ \delta)$, and $a'_k =a_k$ otherwise, will do the trick.
  Indeed, we have $a_ia_j = a'_ia'_j$ as well as $a_i+a_j > a'_i+a'_j$ for small enough $\delta$. This ensures both \ref{eq:ag-lemma1} and \ref{eq:ag-lemma2}.
\end{proof}
An analogous statement for Lemma \ref{lemma:hadamard} has the following form.

\begin{lemma}[Locally improving Hadamard's inequality] Let $M \in \mathbb{R}^{k \times k}$ be a non-singular matrix with SVD $M = U\Sigma V^\T$, and column vectors $c_1$, \dots, $c_k$, for which
\begin{align} \label{equi:hadam-not-eq}
\prod_{i=1}^{k} \|c_i\| > |\det M|.
\end{align}
Then there exists an orthogonal matrix $V'$, a small perturbation of $V$, such that if we denote by $c'_1$, \dots, $c'_k$ the column vectors of $M' = U\Sigma V'^\T$, we have
\begin{align}
\prod_{i=1}^{k} \|c_i\| > \prod_{i=1}^{k} \|c'_i\|.
\end{align}
\end{lemma}
\begin{proof}
We proceed by induction on $k$. For $k=2$, it can be verified directly that for some small $\delta$ (in absolute value) setting $V' = VR_{\delta}$, where $R_{\delta}$ is a 2D rotation matrix by angle $\delta$, achieves what is required.

For the general case, the sharp inequality \eqref{equi:hadam-not-eq} implies that $c_i^\T c_j \neq 0$ for some pair of $i \neq j$. Without loss of generality, let $i=1$, $j=2$. In such case, we consider $V' = VR^{2D}_{\delta}$, where
\begin{align}
R^{2D}_{\delta} =
\begin{pmatrix}
    \diagentry{R_{\delta}}\\
    &\diagentry{\Unit_{k-2}}
\end{pmatrix}
\end{align}
is a block diagonal matrix, in which $R_{\delta}$ is again a $2 \times 2$ rotation matrix. By design, we have $c_i = c'_i$ for $i > 2$. This, along with the fact that $U$ can be set to $\Unit_k$ (isometry does not influence either side of \eqref{equi:hadam-not-eq}), allows for a full reduction to the discussed two-dimensional case.
\end{proof}

It is easy to see that the performed perturbations continuously translate into perturbations of the parameters $\sigma^2_j(\xi)$ and $V_i$ in estimates \eqref{eq:am-gm-used} and \eqref{eq:hadam-used}. Consequently, any non-optimal values of $\sigma^2_j(\xi)$ and $V_i$ can be locally improved. This concludes the proof.

\subsection{Rotational invariances}

Let us start by fleshing out the common elements of the proofs of Propositions \ref{prop:log_invariant} and \ref{prop:elbo_invariant}. In both cases, the encoder and decoder mappings $\Enc_{,U}$, $\Dec_{,U}$ induce joint distributions $p_U(\mb{x}, \mb{z})$, $q_U(\mb{x}, \mb{z})$ described as
\begin{align}
p_U(\mb{x}, \mb{z}) &= p(\mb{z}) p(\mb{x} \mid U^\T \mb{z}) \\
q_U(\mb{x}, \mb{z}) &= q(\mb{x}) q(U^\T \mb{z} \mid \mb{x})
\end{align}

\begin{lemma}\label{lemma:p_preserved} For every $\xi \in X$ we have $p(\xi) = p_U(\xi)$.
\end{lemma}
\begin{proof}
We simply compute
\begin{align*}
p_U(\xi) &= \int p_U(\xi, \mb{z}) \dd \mb{z} \\
&= \int p(\mb{z}) p(\xi \mid U^\T \mb{z}) \dd \mb{z} \\
&= \int p(U\mb{z}) p(\xi \mid \mb{z}) \dd \mb{z} \\
&= \int p(\mb{z}) p(\xi \mid \mb{z}) \dd \mb{z} = p(\xi),
\end{align*}
where in the third equality we used the Change of Variable Theorem to substitute $U\mb{z}$ for $\mb{z}$ (keep in mind that $|\det(U)| = 1$ as $U$ is an orthogonal matrix). In the fourth equality, we used the rotational symmetry of the prior $p(\mb{z})$.
\end{proof}

\proofof{Proposition \ref{prop:log_invariant}}{
This immediately follows from Lemma \ref{lemma:p_preserved}.}

\proofof{Proposition \ref{prop:elbo_invariant}}{
We utilize the full identity from ELBO derivation. For fixed $\xi \in X$ we have \cite{KingmaWelling2014:VAE}
\begin{align}\label{eq:elbo_full}
\mathrm{ELBO} = \KL{q_U(\mb{z} \mid \xi)}{p_U(\mb{z} \mid \xi)} + \log p_U(\xi)
\end{align}
In order to prove invariance of ELBO to the choice of $U$, it suffices to prove invariance of the right-hand side of \eqref{eq:elbo_full}. Due to Proposition \eqref{prop:elbo_invariant} we only need to focus on the KL term. Similarly as in the proof of Lemma \ref{lemma:p_preserved}, we calculate
\begin{align*}
\qquad&\KL{q_U(\mb{z} \mid \xi)}{p_U(\mb{z} \mid \xi)} \\
&= \int q_U(\mb{z} \mid \xi) \log \frac{q_U(\mb{z} \mid \xi)}{p_U(\mb{z} \mid \xi)}\dd \mb{z} \\
&= \int q_U(\mb{z} \mid \xi) \log \frac{q_U(\mb{z} \mid \xi)\cdot  p_U(\xi)}{p_U(\mb{z}) \cdot p_U(\xi \mid \mb{z})}\dd \mb{z} \\
&\stackrel{(3)}{=} \int q(U^\T \mb{z} \mid \xi) \log \frac{q(U^\T \mb{z} \mid \xi)\cdot  p(\xi)}{p(\mb{z}) \cdot p(\xi \mid U^\T \mb{z})}\dd \mb{z} \\
&\stackrel{(4)}{=} \int q(\mb{z} \mid \xi) \log \frac{q(\mb{z} \mid \xi)\cdot  p(\xi)}{p(\mb{Uz}) \cdot p(\xi \mid \mb{z})}\dd \mb{z} \\
&\stackrel{(5)}{=} \int q(\mb{z} \mid \xi) \log \frac{q(\mb{z} \mid \xi)\cdot  p(\xi)}{p(\mb{z}) \cdot p(\xi \mid \mb{z})}\dd \mb{z} \\
&= \int q(\mb{z} \mid \xi) \log \frac{q(\mb{z} \mid \xi)}{p(\mb{z} \mid \xi)}\dd \mb{z} \\
&= \KL{q(\mb{z} \mid \xi)}{p(\mb{z} \mid \xi)},
\end{align*}
where we again used the Change of Variable Theorem in equality (4), rotational symmetry of $p(\mb{z})$ in equality (5), and Lemma \ref{lemma:p_preserved} in equality (3).
}

\subsection{Other proofs}

\proofof{Proposition \ref{prop:equi_decoder}}{
Recall from Lemma \ref{lemma:ortho} that column orthogonality of $M$ is equivalent to $M^\T M$ being a diagonal matrix.

\noindent (b) $\Rightarrow$ (a): Let $M = U\Sigma V^\T$ where $|V|$ is a permutation matrix. Then
\begin{align} \label{eq:svd_rewrite}
M^\T M = V \Sigma^\T U^\T U \Sigma V^\T = V \Sigma' V^\T
\end{align}
where $\Sigma' = \Sigma^\T \Sigma$ is a diagonal matrix. But then $V \Sigma' V^\T$ only permutes the diagonal entries of $\Sigma'$ (and possibly flips their signs). In particular, $V \Sigma' V^\T$ is also diagonal.

\smallskip

\noindent (a) $\Rightarrow$ (b): Let again $M = U\Sigma V^\T$ be some SVD of $M$ and assume $M^\T M = D$ for some diagonal matrix $D$. Since $M$ has $d$ distinct nonzero singular values, $M^\T M$ has $d$ distinct nonzero eigenvalues (diagonal elements). Moreover, these eigenvalues are precisely the squares of the singular values captured by $\Sigma$. Next, if we denote by $P$ the permutation matrix for which $PDP^{-1}$ has decreasing diagonal elements, we can write
\begin{align} \label{eq:perm_sing_vals}
PDP^{-1} = \Sigma^\T \Sigma
\end{align}
Then using \eqref{eq:perm_sing_vals} and the SVD of $M$  similarly as in \eqref{eq:svd_rewrite}, we obtain
\begin{align}
D = M^\T M = V\Sigma^\T \Sigma V^\T = VPDP^{-1}V^\T.
\end{align}
Further, the resulting identity $(VP)D = D(VP)$ implies that columns of $VP$ are eigenvectors of $D$, \ie the canonical basis vectors. Since $VP$ is additionally orthogonal, these eigenvectors are normalized. It follows that $|VP|$ is a permutation matrix and the conclusion follows.

}

\proofof{Proposition \ref{prop:loss_decomp}}{
First, note that for any random variable $\mb{X} \in \mathbb{R}^k$ with $\mathbb{E} \mb{X} = \mb{\mu}$ and a constant $\mb{b} \in \mathbb{R}^k$, the following identity holds
\begin{align}
\E \|\mb{X} - \mb{b}\|^2 = \E \|\mb{X} - \mb{\mu}\|^2 + \|\mb{\mu} - \mb{b}\|^2.
\end{align}
In our case, we set $\mb{X} = \Dec(\Enc(\xi))$, the unbiasedness assumption translates to $\mathbb{E} \mb{X} = \Dec(\mu(\xi))$, and finally we set $\mb{b} = \xi$.

The identity we obtain, is exactly what was required to prove.
}

\section{Experimental details}\label{sec:app:exp}

\subsection{Disentanglement Score}\label{sec:app:mig}
As introduced in the paper, for disentangled representations, single latent variables should be sensitive to individual generating factors and insensitive to all others. To quantify this behavior, for each generating factor $w_i$, all latent variables are evaluated for their sensitivity to $w_i$. The sensitivity difference between the two most responsive variables then reflects both desired properties; the sensitivity of the associated best matching latent variable and also the insensitivity of all others. A set of quantities capturing disentanglement can therefore be described as
\begin{align}
\label{eq:disent}
\mathrm{Disent.}& = \frac{1}{N_{\mathrm{labels}}} \sum_{i=1}^{N} \left( \frac{A_{i,m(i)} - A_{i,s(i)}}{M_i} \right) \\
&\mathrm{for}\quad m(i) = \argmax_l \left(A_{i,l}\right)\\
&\mathrm{for}\quad s(i) = \argmax_{k \neq m(i)} \left(A_{i,k}\right),
\end{align}
where $A_{i, j}$ is some sort of sensitivity measure of latent variable $z_j$ with respect to the generating factor $w_i$ and $M_i$ is a normalization constant, ensuring the summands fall into the interval $(0, 1)$.

The recently proposed Mutual Information Gap (MIG)~\cite{tc-beta-vae} uses the Mutual Information as a measure of how the latent variables depend on the generating factors.
For the normalisation, the entropy of the generating factor is used.
\begin{align}
A_{i,j} &= \mi(w_i, z_j)\\
M_i &= \mathrm{H}(w_i)
\end{align}
For discrete generating factors $\{w_i\}$, the normalization with the entropy $H(w_i)$, binds the MIG to the $(0, 1)$ interval, as expected.
For continuous generating factors on the other side, this does not hold. In fact, differential entropy can be zero or even negative and no good normalization is possible.

To treat this shortcoming, we introduce the slightly modified \emph{Disentanglement score} such that it comprises continuous and discrete variables alike.
Rather than using mutual information measurements, we employ powerful nonlinear regressors and classifiers for the two different classes of latent variables.
The predictability of a generating factor from a given latent coordinate indirectly reflects how much information the two share.

Accordingly, we define the Disentanglement score as in Equation \ref{eq:disent} by defining $A_{i,j}$ as the prediction performance of the regressor/classifier for predicting generating factor $w_i$ from the latent coordinate $z_j$. The normalization factor is then the performance of the best constant classifier/regressor. In case of regression with mean square error, this is simply the standard deviation of the generative factor.

More precisely,
\begin{align}
A_{i,j} = \begin{cases}
   \sqrt{\var(\vphantom{|}w_i)} -  \sqrt{\vphantom{|}\mathrm{mse}_{z_j \rightarrow w_i}}, & \text{for regression}\\
   \text{accuracy}_{z_j \rightarrow w_i}, & \text{for classification}\\
  \end{cases}
\end{align}

and
\begin{align*}
M_i = \begin{cases}
    \sqrt{\vphantom{|}\var(w_i)}, & \text{for regression}.\\
    \mathrm{accuracy}_{z_j \rightarrow w_i}^\text{const}, & \text{for classification}.\\
  \end{cases}
\end{align*}

We used the SciPy implementation of a $k$-nearest-neighbors classifier and regressor with default settings (\eg $k=5$) to measure the Disentanglement Score. The regressor/classifier was trained on 80\% of the test data and evaluated on the remaining 20\%.

\begin{table*}\centering{
\renewcommand{\cellalign}{tl}
\caption{Overview over the used datasets and network architectures. The nonlinearities are only applied in the hidden layers. Biases are used for all datasets.}
\label{tbl:params}
\begin{tabular}{|l|llccc|}
\hline
 & \textbf{\makecell{Optimizer\\(LR)}} & \textbf{Architecture} & \textbf{Latent Dim.} & \textbf{Epochs} & $\mathbf{\beta}$\\\hline\hline
\textbf{dSprites} & AdaGrad & \textbf{Enc:} $1200 - 1200$ (Relu) & $5$ & $50$ & $4$\\
& ($10^{-2}$) & \textbf{Dec:} $1200 - 1200 - 1200$ (Tanh) & & &\\
\textbf{Synth. Lin.} & Adam  & \textbf{Enc:} No hidden Layers (Lin) & $2$ & $600$ & $10^{-4}$\\
& ($10^{-3}$) & \textbf{Dec:} No hidden Layers (Lin) & & &\\
\textbf{Synth. Non-Lin.} & Adam & \textbf{Enc:} $60 - 40 - 20$ (Tanh) & $2$ & $600$ & $10^{-3}$\\
& ($10^{-3}$) & \textbf{Dec:} $60 - 40 - 20$ (Tanh) & & &\\
\textbf{MNIST} & AdaGrad & \textbf{Enc:} $400$ (Relu) & $6$ & $400$ & $1$\\
& ($10^{-2}$) & \textbf{Dec:} $500 - 500$ (Tanh) & & &\\
\textbf{fMNIST} & AdaGrad & \textbf{Enc:} $400$ (Relu) & $6$ & $500$ & $1$\\
& ($10^{-2}$) & \textbf{Dec:} $500 - 500$ (Tanh) & & & \\
\textbf{CelebA} & Adam & Conv/Deconv: $[$number of kernels, kernel size, stride$]$ & $32$ & $50$ & $4$\\
 & ($10^{-4}$) & \makecell{\textbf{Enc:} $[[32,4,2],[32,4,2], [64,4,2], [64,4,2]]$ \\ (Relu)} &  &  & \\
&  & \makecell{\textbf{Dec:} $[[64], [64,4,2], [32,4,2], [32,4,2], [3,4,2]]$ \\(Relu), first layer is connecting MLP} & & & \\
 \hline
\end{tabular}}
\end{table*}

\subsection{DtO via Integer Programming}\label{sec:app:procrutes}
The \emph{Distance to Orthogonality} (DtO) describes the Frobenius norm of the difference between a matrix $V$ and its closest signed permutation matrix $P(V)$.
Using mixed-integer linear programming (MILP) formulation, we find the closest permutation matrix as the optimum $P^*$ of the following optimization problem
\begin{align}
\min_{P} &\sum_{i, j} | V_{i, j} - P_{i, j} | & \\ \nonumber
\mathrm{s.t.} \quad & P_{i,j} \in \{-1, 0, 1\} \quad &\forall\ (i, j) \\ \nonumber
& \sum_i |P_{i,j}| = 1 \quad &\forall\ j \\ \nonumber
& \sum_j |P_{i,j}| = 1 \quad &\forall\ i.
\end{align}
Producing a clean MILP formulation, with purely linear objective and binary integer values, can be achieved with a standard technique; introducing new variables.
In particular, we set
\begin{align}
&P_{i, j} = P_{i, j}^+ - P_{i, j}^- \\ \nonumber
&\mathrm{for} \quad P_{i, j}^+, P_{i, j}^- \in \{0, 1\}  \quad \forall\ (i, j)
\end{align}
and introduce (continuous) variables for the differences $V_{i, j} - P_{i, j}$
\begin{align}
V_{i, j} - P_{i, j} &\leq D_{i, j} \quad &\forall\ (i, j) \\
P_{i, j} - V_{i, j} &\leq D_{i, j} \quad &\forall\ (i, j). \nonumber
\end{align}
The final formulation then is

\begin{align}
\min_{P} \sum_{i, j} D_{i, j}\quad\quad& \\ \nonumber
\mathrm{s.t.} \quad (P_{i, j}^+ - P_{i, j}^-) - V_{i, j} &\leq D_{i, j} \quad &\forall\ (i, j) \\ \nonumber
V_{i, j} - (P_{i, j}^+ - P_{i, j}^-) &\leq D_{i, j} \quad &\forall\ (i, j)\\ \nonumber
\sum_i \left( P^+_{i,j} + P^-_{i,j} \right) &= 1 \quad &\forall\ j \\ \nonumber
\sum_j \left( P^+_{i,j} + P^-_{i,j} \right) &= 1 \quad &\forall\ i.
\end{align}

\subsection{$\beta$-VAE with Full Covariance Matrix}\label{sec:app:fullVAE}
In the derivation of the VAE loss function, the approximate posterior is set to be a multivariate normal distribution with a diagonal covariance matrix. The claim of the paper is that this diagonalization is responsible for the orthogonalization.
As one of the control experiments in Section \ref{sec:experiments} we also implemented VAE with a full covariance matrix.

Two issues now need to be addressed; computing KL divergence in closed form and adapting the reparametrization trick.
Regarding the former, the sought identity is
\begin{align}
\KL{\mathcal{N}(\mu, \Sigma)}{\mathcal{N}(0, \mathcal{I}_k)} = \\
\frac{1}{2} \left( \|\mu \|^2 + \tr\left( \Sigma \right) - \log \det \Sigma - k \right).
\end{align}
As for the reparametrization trick, if $\varepsilon \sim \mathcal{N}(0, \mathcal{I}_k)$, it is easy to check that
\begin{align}
\mu + \Sigma^{1/2}\varepsilon \sim \mathcal{N}(\mu, \Sigma),
\end{align}
where $\Sigma = \Sigma^{1/2}\cdot \left(\Sigma^{1/2}\right)^\T$ is the unique Cholesky decomposition of the positive definite matrix $\Sigma$.

\begin{figure*}[!htb]
\begin{center}
\includegraphics[width=\textwidth]{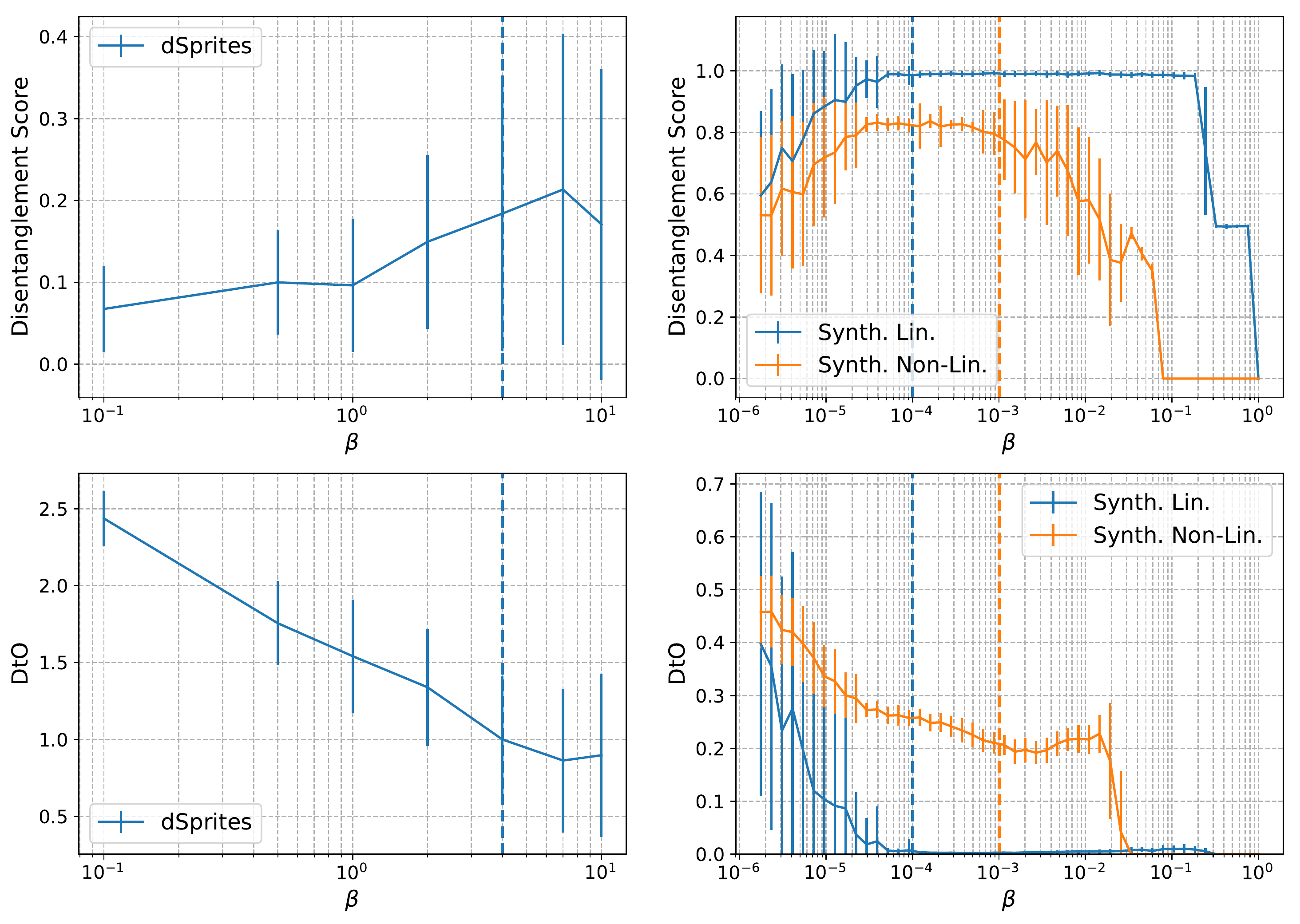}
\caption{The $\beta$ hyper-parameter in the $\beta$-VAE allows to trade-off reconstruction error and the KL loss such that the desired amount of disentanglement is achieved. The plots show the Disentanglement Score (top) and the DtO (bottom) for dSprites (left) and synthetic datasets (right). The dashed lines indicate the parameter chosen for the experiments.
  \label{fig:app:sweep}}
\end{center}
\end{figure*}

\subsection{Network Details and Training} \label{sec:app:training}
Table \ref{tbl:params} contains the training parameters used for the different architectures. The listed latent dimension is chosen to be the number of independent generating factors, if applicable, and chosen large enough to ensure decent reconstruction loss on all architectures.

All reported numbers are calculated using a previously unseen test dataset. To facilitate this, we split the whole datasets randomly into three parts for training, evaluation and test (containing $80\,\%$, $10\,\%$ and $10\,\%$ of all samples respectively). During development, we use the evaluation dataset, for the final reports we use the test dataset.

\subsection{Synthetic Datasets} \label{sec:app:synth_datasets}
The linear synthetic dataset is generated with a transformation $f_\mathrm{lin} \colon \mathbb{R}^2 \rightarrow \mathbb{R}^3$, mapping a unit square $V = [0, 1]^2$ to a 3-dimensional space. The transformation can be decomposed into:
\begin{enumerate}
\item stretching along one axis by a fixed factor of 2, 
\item trivial embedding into $\mathbb{R}^3$,
\item rotation of $45^\circ$ along the line containing the vector $\left( 1, -1, 1\right)$.
\end{enumerate}

For the non-linear dataset, the transformation $f_\mathrm{non-lin} \colon \mathbb{R}^2 \rightarrow \mathbb{R}^{6}$ is realized by a random initialization of a MLP with one hidden layer (width $10$), biases and $\mathrm{tanh}$ nonlinearitites.

Both datasets consist of $50000$ samples.

\section{Additional Experiments}\label{sec:app:addexp}

\subsection{Dependence of Disentanglement Score and DtO on $\beta$}\label{sec:app:ortho}
The choice of $\beta$ depends on the achievable Disentanglement Score. Figure \ref{fig:app:sweep} shows a more thorough analysis of the dependence of both the Disentanglement Score and the DtO.
For too small values of $\beta$, the effect of the KL term (and thus the orthogonalization) is negligible.
In the other extreme case, too large values values of $\beta$ result in overpruning, such that the number of active latent coordinates drops below the number of generating factors.

\subsection{Degenerate case }\label{sec:app:degeneracy}
\begin{figure}[h!]
\begin{center}
\includegraphics[width=\columnwidth]{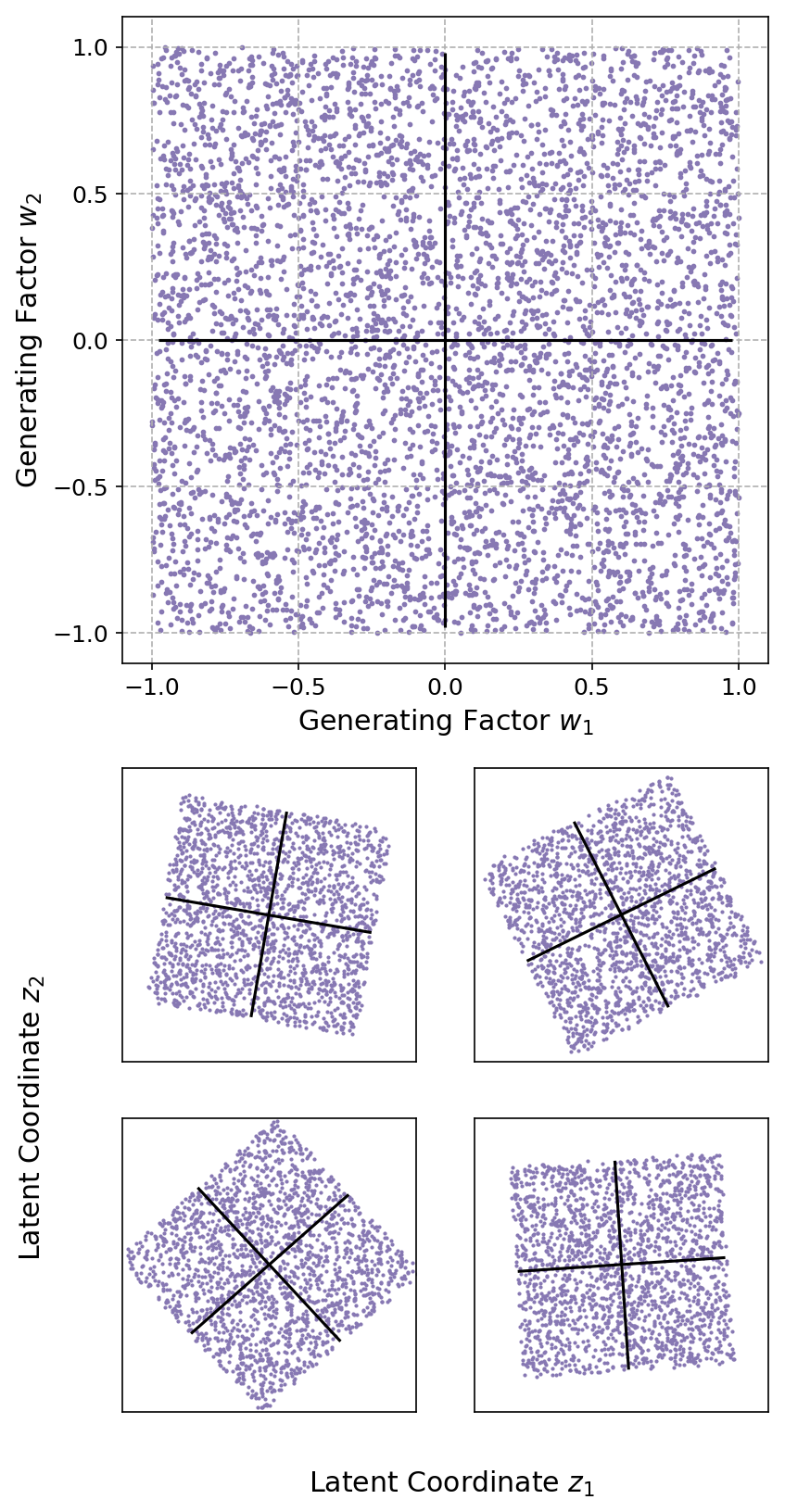}
\caption{For strong degeneracy, e.g.\ in the synthetic dataset with the two generating factors $w_1$ and $w_2$ on equal, uniform scale (top), the linear $\beta$-VAE generates arbitrarily rotated latent representations (bottom) here for the linear synthetic dataset. \label{fig:app:degen_case}}
\end{center}
\end{figure}

Proposition \ref{prop:equi_decoder} insists that the locally linearized decoder
 have distinct singular values, otherwise orthogonality of the column vectors does not translate into preserving axes. Here, we design an experiment showing, that this condition is also relevant in practice.

The dataset in question will be a version of the linear synthetic task where the generating factors have the same scaling, as visualized in the upper plot of Figure~\ref{fig:app:degen_case}.
Note that any linear encoder applying a simple rotation has both orthogonal columns and equal singular values. But it does not respect the alignment of the original square, as it does not meet the assumptions of Proposition \ref{prop:equi_decoder}.

Behavior of the $\beta$-VAE with a linear encoder/decoder network is consistent with this.
The bottom part of Figure~\ref{fig:app:degen_case} shows $\beta$-VAE latent representations of four random restarts; they expose random alignments.
The same effect results in high variances for both the Disentanglement Score and the DtO, as shown in Table \ref{tbl:degen}.

This degeneracy also occurs for PCA. It is easy to check that \emph{any} projection of a unit square on a line \emph{has equal variance}. Hence the greedy PCA algorithm has no preference over which alignment to choose, and the practical choice of alignment is implementation dependent.

This insight reinforces our point that $\beta$-VAE (just like PCA) looks for sources of variance rather than for statistical independence.

We can also see in Table \ref{tbl:degen}, that the degeneracy disappears even for small rescaling of the ground truth factors. Since $\beta$-VAE promotes normalized latent representations (zero mean, unit variance), the singular values will no longer be equal and the right alignment is found. The same is true for PCA.

\begin{table}\centering{
\caption{Overview of Disentanglement Score and DtO for different ratios of importance between the generating factors for the Synth.~Lin.\ task. A ratio of $1.2$ means one generating factor is scaled by $1.2$. \label{tbl:degen}}
\begin{tabular}{|c|ccc|}\hline
\textbf{Ratio} & $1.0$ & $1.2$ & $1.5$ \\\hline\hline
\textbf{Disent.} & $0.51 \pm 0.28$ & $0.76 \pm 0.25$ & $0.98 \pm 0.06$\\
\textbf{DtO} & $0.49 \pm 0.32$ & $0.20 \pm 0.24$ & $0.01 \pm 0.06$
\\\hline
\end{tabular}
}
\end{table}

\begin{figure}[h!]
\begin{center}
\includegraphics[width=\columnwidth]{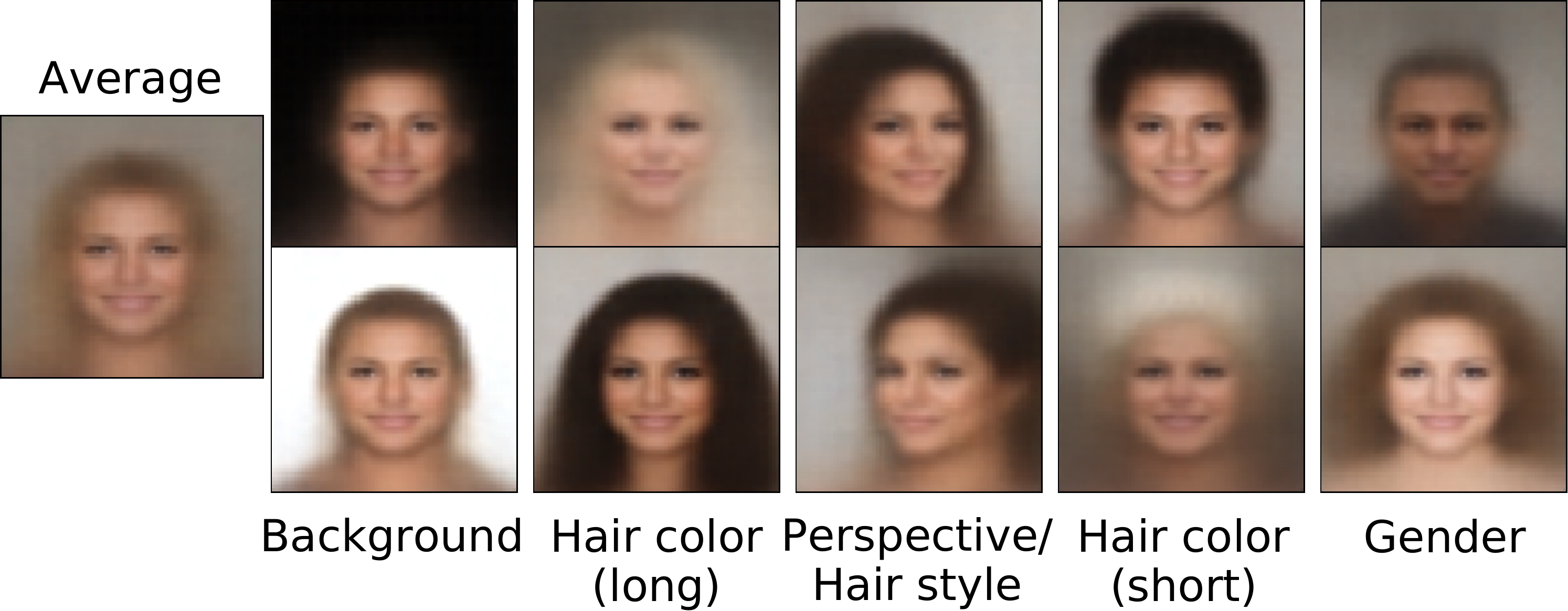}
\caption{For strong degeneracy, e.g.\ in the synthetic dataset with the two generating factors $w_1$ and $w_2$ on equal, uniform scale (top), the linear $\beta$-VAE generates arbitrarily rotated latent representations (bottom) here for the linear synthetic dataset. \label{fig:app:eigenfaces}}
\end{center}
\end{figure}
\subsection{Non-Linear VAE Eigenfaces}
\label{sup:sec:eigenfaces}
In order to highlight the connection with PCA, we use $\beta$-VAE to produce a non-linear version of the classical eigenfaces~\cite{turk1991eigenfaces} on the CelebA dataset~\cite{liu2015faceattributes}.
Fig \ref{fig:app:eigenfaces} shows a discrete latent traversal. 
 
Starting from the latent representation $z_\mathrm{mean}$ of the mean face (over $300$ randomly selected datapoints) we feed $\{ \mathbf{z}_\mathrm{mean} \pm \alpha \mathrm{\mathbf{e}}_{i}\}$ through the decoder, where $\mathrm{\mathbf{e}}_i$ are the canonical base vectors. 
Particularly, we chose $i$ covering the first $5$ latent coordinates, sorted by the mean $\sigma_j$. 
The parameter $\alpha = 2.5$ was empirically chosen to be on near the tails of the distribution over $\mathbf{z}^k$.

We can see that unlike classical eigenfaces that mostly reflect photometric properties, the `nonlinear eigenfaces'' capture also semantic features of the data. Note also that the ordering of the `principal components' by the mean values of $\sigma_j$ is naturally justified by our work. As was illustrated in Sec. \ref{sec:isolatingsourcesofvariance} of the paper,
the first $\beta$-VAE `principle components' also focus on characteristics with high impact on the reconstruction loss (\ie capture the most variance), 

Details about the architecture used are listed in Tab. \ref{tbl:params}.

{\small
\bibliographystyle{ieee_fullname}
\bibliography{bibliography}
}
\end{document}